\documentclass{article}

\pdfoutput=1
\usepackage[preprint]{neurips2023}

\usepackage{amsmath,amsfonts,bm}

\def\eqref#1{equation~\ref{#1}}

\def\1{\bm{1}}

\def\vx{{\bm{x}}}
\def\vy{{\bm{y}}}
\def\vz{{\bm{z}}}

\DeclareMathAlphabet{\mathsfit}{\encodingdefault}{\sfdefault}{m}{sl}
\SetMathAlphabet{\mathsfit}{bold}{\encodingdefault}{\sfdefault}{bx}{n}

\def\gA{{\mathcal{A}}}

\def\gF{{\mathcal{F}}}

\def\gM{{\mathcal{M}}}
\def\gN{{\mathcal{N}}}

\def\gS{{\mathcal{S}}}

\def\gX{{\mathcal{X}}}

\usepackage{natbib}
\setcitestyle{numbers,square}
\usepackage{wrapfig}
\usepackage{tikz}
\usetikzlibrary{positioning, shapes.geometric}
\usepackage[utf8]{inputenc} 
\usepackage[T1]{fontenc}    
\usepackage{hyperref}       
\usepackage{url}            
\usepackage{booktabs}       
\usepackage{amsfonts}       
\usepackage{nicefrac}       
\usepackage{microtype}      
\usepackage{xcolor,dsfont,pifont}         
\usepackage[most]{tcolorbox}

\usepackage{amsmath,amsfonts,comment,bbm,setspace}

\usepackage{hyperref,array}
\usepackage{url}
\usepackage{amsmath,amsfonts,amssymb, amsthm,dsfont}
\usepackage{algorithm, algorithmic}
\usepackage{graphicx, epstopdf}
\usepackage{apxproof}
\usepackage{caption}
\usepackage{subcaption}
\usepackage{xspace}

\newcommand{\qwhittle}{\texttt{Q-Whittle}\xspace}
\newcommand{\qwhittleLFA}{\texttt{Q-Whittle-LFA}\xspace}
\newcommand{\qwhittleNFA}{\texttt{Neural-Q-Whittle}\xspace}

\newcommand{\cX}{\mathcal{X}}

\newtheorem{assumption}{Assumption}
\newtheorem{lemma}{Lemma}
\newtheorem{remark}{Remark}
\newtheorem{theorem}{Theorem}

\newtheorem{defn}{Definition}

\title{Finite-Time Analysis of Whittle Index based Q-Learning for Restless Multi-Armed Bandits
with Neural Network Function Approximation}

\author{%
  Guojun~Xiong, ~Jian~Li \\
  Stony Brook University\\
  \texttt{\{guojun.xiong,jian.li.3\}@stonybrook.edu} \\
}

\begin{document}

\maketitle

\begin{abstract}

Whittle index policy is a heuristic to the intractable restless multi-armed bandits (RMAB) problem. Although it is provably asymptotically optimal, finding Whittle indices remains difficult.  In this paper, we present \qwhittleNFA, a Whittle index based Q-learning algorithm for RMAB with neural network function approximation, which is an example of  nonlinear two-timescale stochastic approximation with Q-function values updated on a faster timescale and Whittle indices on a slower timescale.  
Despite the empirical success of deep Q-learning, the non-asymptotic convergence rate of \qwhittleNFA, which couples neural networks with two-timescale Q-learning largely remains unclear.  This paper provides a finite-time analysis of \qwhittleNFA, where data are generated from a Markov chain, and Q-function is approximated by a ReLU neural network. Our analysis leverages a Lyapunov drift approach to capture the evolution of two coupled parameters, and the nonlinearity in value function approximation further requires us to characterize the approximation error.  Combing these provide \qwhittleNFA  with $\mathcal{O}(1/k^{2/3})$ convergence rate, where $k$ is the number of iterations.


\end{abstract}

\section{Introduction}\label{sec:intro}

We consider the restless multi-armed bandits (RMAB) problem \cite{whittle1988restless}, where the decision maker (DM) repeatedly activates $K$ out of $N$ arms at each decision epoch. Each arm is described by a Markov decision process (MDP) \cite{puterman1994markov}, and evolves stochastically according to two different transition kernels, depending on whether the arm is activated or not. Rewards are generated with each transition. Although RMAB has been widely used to study constrained sequential decision making problems \cite{bertsimas2000restless,meshram2016optimal,borkar2017index,yu2018deadline,killian2021beyond,mate2021risk,mate2020collapsing,jiang2023online}, it is notoriously intractable due to the explosion of state space \cite{papadimitriou1994complexity}. A celebrated heuristic is the Whittle index policy \cite{whittle1988restless}, which computes the Whittle index for each arm given its current state as the cost to pull the arm.  Whittle index policy then activates the $K$ highest indexed arms at each decision epoch, and is provably asymptotically optimal \cite{weber1990index}.

However, the computation of Whittle index requires full knowledge of the underlying MDP associated with each arm, which is often unavailable in practice. To this end, many recent efforts have focused on learning Whittle indices for making decisions in an online manner.  First, model-free reinforcement learning (RL) solutions have been proposed  \cite{borkar2018reinforcement,fu2019towards,wang2020restless,biswas2021learn,killian2021q,xiong2022reinforcement,xiong2022Nips,xiong2022reinforcementcache,xiong2022indexwireless,avrachenkov2022whittle}, among which \cite{avrachenkov2022whittle} developed a Whittle index based Q-learning algorithm, which we call \qwhittle for ease of exposition, and provided the first-ever rigorous asymptotic analysis.  However, \qwhittle suffers from slow convergence since it only updates the Whittle index of a specific state when that state is visited. In addition, \qwhittle needs to store the Q-function values for all state-action pairs, which limits its applicability only to problems with small state space.  Second, deep RL methods have been leveraged to predict Whittle indices via training neural networks \cite{nakhleh2021neurwin,nakhleh2022deeptop}. Though these methods are capable of dealing with large state space, there is no asymptotic or finite-time performance guarantee.  Furthermore, training neural networks requires to tuning hyper-parameters. This introduces an additional layer of complexity to predict Whittle indices.  Third, to address aforementioned deficiencies, \cite{xiong2023whittle} proposed \qwhittleLFA by coupling \qwhittle 
with linear function approximation and provided a finite-time convergence analysis.  One key limitation of \qwhittleLFA is the unrealistic assumption that all data used in \qwhittleLFA are sampled i.i.d. from a fixed stationary distribution.  

To tackle the aforementioned limitations and inspired by the empirical success of deep Q-learning in numerous applications, we develop \qwhittleNFA, a Whittle index based Q-learning algorithm with \textit{neural network function approximation under Markovian  observations}.  Like \cite{avrachenkov2022whittle,xiong2023whittle}, the updates of Q-function values and Whittle indices form a two-timescale stochastic approximation (2TSA) with the former operating on a faster timescale and the later on a slower timescale.  Unlike \cite{avrachenkov2022whittle,xiong2023whittle}, our \qwhittleNFA uses a deep neural network with the ReLU activation function to approximate the Q-function.  {However, Q-learning with neural network function approximation can in general diverge \cite{achiam2019towards},} and the theoretical convergence of Q-learning with neural network function approximation has been limited to special cases such as fitted Q-iteration with i.i.d. observations \cite{fan2020theoretical}, which fails to capture the practical setting of Q-learning with neural network function approximation.  

In this paper, we study the non-asymptotic convergence of \qwhittleNFA with data generated from a Markov decision process. Compared with recent theoretical works for Q-learning with neural network function approximation \cite{cai2023neural,fan2020theoretical,xu2020finite}, our \qwhittleNFA involves a two-timescale update between two coupled parameters, i.e., Q-function values and Whittle indices.  This renders existing finite-time analysis in \cite{cai2023neural,fan2020theoretical,xu2020finite} not applicable to our \qwhittleNFA due to the fact that \cite{cai2023neural,fan2020theoretical,xu2020finite} only contains a single-timescale update on Q-function values.  Furthermore, \cite{cai2023neural,fan2020theoretical,xu2020finite} required an additional projection step for the update of parameters of neural network function so as to guarantee the boundedness between the unknown parameter at any time step with the initialization.  
This in some cases is impractical.
Hence, a natural question that arises is 
\vspace{-0.1in}
\begin{tcolorbox}[colback=white!5!white,colframe=white!75!white]
\textit{Is it possible to provide a non-asymptotic convergence rate analysis of \qwhittleNFA with two coupled parameters updated in two timescales under Markovian observations without the extra projection step?}
\end{tcolorbox}
\vspace{-0.05in}

The theoretical convergence guarantee of two-timescale Q-learning with neural network function approximation under Markovian observations remains largely an open problem, and in this paper, we provide an affirmative answer to this question.  Our main contributions are summarized as follows: 

$\bullet$ We propose \qwhittleNFA, a novel Whittle index based Q-learning algorithm with neural network function approximation for RMAB.  Inspired by recent work on TD learning \cite{srikant2019finite}  and Q-learning \cite{chen2019performance} with linear function approximation, our \qwhittleNFA removes the additional impractical projection step in the neural network function parameter update.

$\bullet$ We establish the first finite-time analysis of \qwhittleNFA under Markovian observations. 
Due to the two-timescale nature for the updates of two coupled parameters (i.e., Q-function values and Whittle indices) in \qwhittleNFA, we focus on the convergence rate of these parameters rather than the convergence rate of approximated Q-functions as in \cite{cai2023neural,fan2020theoretical,xu2020finite}.  Our key technique is to view \qwhittleNFA as a 2TSA for finding the solution of suitable nonlinear equations.  
Different from recent works on finite-time analysis of a general 2TSA \cite{doan2021finite} or with linear function approximation \cite{xiong2023whittle}, 
the nonlinear parameterization of Q-function in \qwhittleNFA under Markovian observations imposes significant difficulty in finding the global optimum of the corresponding nonlinear equations.  To mitigate this, we first approximate the original neural network function with a collection of local linearization and focus on finding a surrogate Q-function in the neural network function class that well approximates the optimum.  Our finite-time analysis then requires us to consider two Lyapunov functions that carefully characterize the coupling between iterates of Q-function values and Whittle indices, with one Lyapunov function defined with respect to the true neural network function, and the other defined with respect to the locally linearized neural network function.  We then characterize the errors between these two Lyapunov functions.  Putting them together, we prove that \qwhittleNFA  achieves a convergence in expectation at a rate $\mathcal{O}(1/k^{2/3})$, where $k$ is the number of iterations. 

$\bullet$ Finally, we conduct experiments to validate the convergence performance of \qwhittleNFA, and verify the sufficiency of our proposed condition for the stability of  \qwhittleNFA.


\section{Preliminaries}

\textbf{RMAB.} We consider an infinite-horizon average-reward RMAB with each arm $n\in\gN$ described by a unichain MDP \cite{puterman1994markov} $\gM_n:=(\gS, \gA, P_n, r_n)$, where $\gS$ is the state space with cardinality $S<\infty$, $\gA$ is the action space with cardinality $A$, $P_n(s^\prime|s, a)$ is the transition probability of reaching state $s^\prime$ by taking action $a$ in state $s$, and $r_n(s,a)$ is the reward associated with state-action pair $(s,a)$.  At each time slot $t$, the DM activates $K$ out of $N$ arms. 
Arm $n$ is ``active'' at time $t$ when it is activated, i.e., $A_n(t)=1$; otherwise, arm $n$ is ``passive'', i.e., $A_n(t)=0$. Let $\Pi$ be the set of all possible policies for RMAB, and $\pi\in\Pi$ is a feasible policy, satisfying $\pi: \gF_t \mapsto \gA^{N}$, where $\gF_t$ is the sigma-algebra generated by random variables $ \{S_n(h), A_n(h): \forall n\in\gN, h \leq t \}$.  The objective of the DM is to maximize the expected long-term average reward subject to an instantaneous constraint that only $K$ arms can be activated at each time slot, i.e., 
\begin{align}\label{eq:obj-cons}
\text{RMAB:}\quad\max_{\pi\in\Pi} & ~~\liminf_{T\rightarrow \infty} \frac{1}{T} \mathbb{E}_\pi \left(\sum_{t=0}^{T}\sum_{n=1}^{N}  r_n(t)\right), \quad 
\mbox{s.t.}~\sum_{n=1}^{N} A_n(t)=K,\quad \forall t.
\end{align}
\textbf{Whittle Index Policy.} It is well known that RMAB~(\ref{eq:obj-cons}) suffers from the curse of dimensionality \cite{papadimitriou1994complexity}.  To address this challenge, Whittle \cite{whittle1988restless} proposed an index policy through decomposition.  Specifically, Whittle relaxed the constraint in~(\ref{eq:obj-cons}) to be satisfied on average and obtained a unconstrained problem: $\max_{\pi\in\Pi} \liminf_{T\rightarrow\infty}\frac{1}{T}\mathbb{E}_\pi\sum_{t=1}^T\sum_{n=1}^{N} \{r_n(t)+\lambda(1- A_n(t))\}$, {where $\lambda$ is the Lagrangian multiplier associated with the constraint.}  The key observation of Whittle is that this problem can be decomposed and its solution is obtained by combining solutions of $N$ independent problems via solving the associated dynamic programming (DP): $V_n(s)=\max_{a\in\{0,1\}}Q_n(s,a), \forall n\in\mathcal{N},$ where 
\begin{align}\label{eq:Q_value}
Q_n(s,a)\!+\!{\beta}\!=\!a\Big(\!r_n(s,\!a)\!+\!\!\sum_{s^\prime}p_n(s^\prime |s,\!1)V_n(s^\prime)\!\Big)\!\!+\!(1\!-\!a)\Big(\!r_n(s,\!a)\!\!+\!\lambda\!\!+\!\!\sum_{s^\prime}p_n(s^\prime |s,\!0)V_n(s^\prime)\!\Big), 
\end{align} 
where $\beta$ is unique and equals to the maximal long-term average reward of the unichain MDP, and $V_n(s)$ is unique up to an additive constant, both of which depend on the Lagrangian multiplier $\lambda.$ The optimal decision $a^*$ in state $s$ then is the one which maximizes the right hand side of the above DP. The Whittle index associated with state $s$ is defined as the value $\lambda_n^*(s)\in\mathbb{R}$ such that actions  $0$ and $1$ are equally favorable in state $s$ for arm $n$ \cite{avrachenkov2022whittle,fu2019towards}, satisfying
\begin{align}\label{eq:Whittle_index}
\lambda_n^*(s) := r_n(s,1)+\sum_{s^\prime}p_n(s^\prime |s,1)V_n(s^\prime) -r_n(s,0)-\sum_{s^\prime}p_n(s^\prime |s,0)V_n(s^\prime).
\end{align}
Whittle index policy then activates $K$ arms with the largest Whittle indices at each time slot. Additional discussions are provided in Section~\ref{sec:whittle-app} in supplementary materials.

\textbf{Q-Learning for Whittle Index.}
Since the underlying MDPs are often unknown, \cite{avrachenkov2022whittle} proposed \qwhittle, a tabular Whittle index based Q-learning algorithm, where the updates of Q-function values and Whittle indices form a 2TSA, with the former operating on a faster timescale for a given $\lambda_n$ and the later on a slower timescale.  Specifically, the Q-function values for $\forall n\in\gN$ are updated as
\begin{align}\label{eq: Convention_Q_update}
     Q_{n,k+1}(s,a)&:=Q_{n,k}(s,a)+\alpha_{n,k} \mathds{1}_{\{S_{n,k}=s,A_{n,k}=a\}}\Big(r_n(s,a)+(1-a)\lambda_{n,k}(s)\nonumber\\
     &\qquad\qquad\qquad\qquad\qquad+\max_{a}Q_{n,k}(S_{n,k+1},a)-I_n(Q_k)-Q_{n,k}(s,a)\Big),
\end{align}
{where $I_n(Q_k)=\frac{1}{2S}\sum_{s\in\gS}(Q_{n,k}(s,0)+Q_{n,k}(s,1))$} is standard in the relative Q-learning for long-term average MDP setting \cite{abounadi2001learning}, which differs significantly from the discounted reward setting 
\cite{puterman1994markov,abounadi2001learning}.  $\{\alpha_{n,k}\}$ is a step-size sequence satisfying $\sum_k\alpha_{n,k}=\infty$ and $\sum_k\alpha_{n,k}^2<\infty$.

Accordingly, the Whittle index is updated as
\begin{align}\label{eq:convention_lambda_update}
    \lambda_{n,k+1}(s)=\lambda_{n,k}(s)+\eta_{n,k}(Q_{n,k}(s,1)-Q_{n,k}(s,0)),
\end{align}
with the step-size sequence $\{\eta_{n,k}\}$ satisfying $\sum_k\eta_{n,k}=\infty$, $\sum_k\eta_{n,k}^2<\infty$ and $\eta_{n,k}=o(\alpha_{n,k})$.  The coupled iterates~(\ref{eq: Convention_Q_update}) and~(\ref{eq:convention_lambda_update}) form a 2TSA, and \cite{avrachenkov2022whittle} provided an asymptotic convergence analysis.



\begin{algorithm}[t]
\caption{\qwhittleNFA: Neural Q-Learning for Whittle Index}
\label{Algorithm1}
\begin{algorithmic}[1]
\STATE \textbf{Input}: $\pmb{\phi}(s,a)$ for $\forall s\in\mathcal{S}, a\in\gA$, and learning rates $\{\alpha_k\}_{k=1,\ldots, T}$, $\{\eta_k\}_{k=1,\ldots,T}$
\STATE \textbf{Initialization}: $b_r\sim\text{Unif}(\{-1,1\}), \mathbf{w}_{r,0}\sim\mathcal{N}(\pmb{0}, \mathbf{I}_d/d),\forall r\in[1,m]$ and
$\lambda(s)=0$, $\forall s\in\mathcal{S}$
\FOR{$s\in\mathcal{S}$}
\FOR{ $k=1, \ldots, T$}
\STATE Sample $(S_k, A_k, S_{k+1})$ according to the $\epsilon$-greedy policy;
\STATE $\Delta_k\!=\!r(S_k,A_k)+(1-A_k)\lambda_k(s)+\max_a f(\pmb{\theta}_k;\pmb{\phi}(S_{k+1},a))-I(\pmb{\theta}_k)\!-\!f(\pmb{\theta}_k;\pmb{\phi}(S_k,A_k))$;
\STATE $\pmb{\theta}_{k+1}=\pmb{\theta}_{k}+\alpha_k\Delta_k\nabla_{\pmb{\theta}}f(\pmb{\theta}_k;\pmb{\phi}(S_k,A_k))$;
\STATE  $\lambda_{k+1}(s)=\lambda_{k}(s)+\eta_k(f( \pmb{\theta}_k;\pmb{\phi}(s,1))-f(\pmb{\theta}_k;\pmb{\phi}(s,0) )$;
\ENDFOR
\ENDFOR
\STATE \textbf{Return:} $\lambda(s), \forall s\in\mathcal{S}$.
\end{algorithmic}
\end{algorithm}

\section{Neural Q-Learning for Whittle Index}
A closer look at~(\ref{eq:convention_lambda_update}) reveals that \qwhittle only updates the Whittle index of a specific state when that state is visited. This makes \qwhittle suffers from slow convergence. In addition, \qwhittle needs to store the Q-function values for all state-action pairs, which limits its applicability only to problems with small state space.  To address this challenge and inspired by the empirical success of deep Q-learning, we develop \qwhittleNFA through coupling \qwhittle with neural network function approximation by using low-dimensional feature mapping and leveraging the strong representation power of neural networks. For ease of presentation, we drop the subscript $n$ in~(\ref{eq: Convention_Q_update}) and~(\ref{eq:convention_lambda_update}), and discussions in the rest of the paper apply to any arm $n\in\gN$.

Specifically, given a set of basis functions $\phi_\ell: \gS\times \gA\mapsto \mathbb{R}, \forall \ell=1,\cdots, d$ with $d\ll SA$, the approximation of Q-function ${Q}_{\pmb{\theta}}(s,a)$ parameterized by a unknown weight vector $\pmb{\theta}\in\mathbb{R}^{md}$, is given by ${Q}_{\pmb{\theta}}(s,a)=f(\pmb{\theta};\pmb{\phi}(s,a)), ~\forall s\in\gS, a\in\gA,$ where $f$ is a nonlinear neural network function 
parameterized by $\pmb{\theta}$ and $\pmb{\phi}(s,a)$, with  $\pmb{\phi}(s,a)=(\phi_1(s,a), \cdots,\phi_d(s,a))^\intercal$. The feature vectors are assumed to be linearly independent and are normalized so that $\|\pmb{\phi}(s,a)\|\leq 1, \forall s\in\gS, a\in\gA$ .
In particular, we parameterize the Q-function by using a two-layer neural network \cite{cai2023neural,xu2020finite}
\begin{align}\label{eq:neural_Q}
    f(\pmb{\theta}; \pmb{\phi}(s,a)):=\frac{1}{\sqrt{m}}\sum_{r=1}^m b_r\sigma(\mathbf{w}_r^\intercal\pmb{\phi}(s,a)),
\end{align}
where $\pmb{\theta}=(b_1,\ldots,b_m,\mathbf{w}_1^\intercal, \ldots, \mathbf{w}_m^\intercal)^\intercal$ with $b_r\in\mathbb{R}$ and $\mathbf{w}_r\in\mathbb{R}^{d\times 1}, \forall r\in[1,m]$. $b_r, \forall r$ are uniformly initialized in $\{-1, 1\}$ and $w_r, \forall r$ are initialized as a zero mean Gaussian distribution according to $\mathcal{N}(\pmb{0}, \mathbf{I}_d/d)$. During training process, only $\mathbf{w}_r, \forall r$ are updated while $b_r, \forall r$ are fixed as the random initialization. Hence, we use $\pmb{\theta}$ and $\mathbf{w}_r, \forall r$ interchangeably throughout this paper.
$\sigma(x)=\max(0,x)$ is the rectified linear unit (ReLU) activation function\footnote{The finite-time analysis of Deep Q-Networks (DQN) \cite{cai2023neural,fan2020theoretical,xu2020finite} and references therein focuses on the ReLU activation function, as it has certain properties that make the analysis tractable. ReLU is piecewise linear and non-saturating, which can simplify the mathematical analysis.
Applying the same analysis to other activation functions like the hyperbolic tangent (tanh) could be more complex, which is out of the scope of this work.}.

Given (\ref{eq:neural_Q}), we can rewrite the Q-function value updates in (\ref{eq: Convention_Q_update}) as
\begin{align} \label{eq:theta_update1}
    \pmb{\theta}_{k+1}&=\pmb{\theta}_k+\alpha_k \Delta_k\nabla_{\pmb{\theta}}f(\pmb{\theta}_k;\pmb{\phi}(S_k, A_k)),
\end{align}
with  $\Delta_k$ being the temporal difference (TD) error {defined as
$\Delta_k:=r(S_k,A_k)+(1-A_k)\lambda_k(s)-I( \pmb{\theta}_k)+\max_{a}f(\pmb{\theta}_k; \pmb{\phi}(S_{k+1},a))-f(\pmb{\theta}_k; \pmb{\phi}(S_{k},A_k))$,} where $I(\pmb{\theta}_k)=\frac{1}{2S}\sum_{s\in\gS}[f(\pmb{\theta}_k; \pmb{\phi}(s,0))+f(\pmb{\theta}_k; \pmb{\phi}(s,1))]$. 
Similarly, the Whittle index update~(\ref{eq:convention_lambda_update}) can be rewritten as
\begin{align}    \label{eq:lambda_update1}
    \lambda_{k+1}(s)&=\lambda_{k}(s)+\eta_k(f(\pmb{\theta}_k; \pmb{\phi}(s,1))-f(\pmb{\theta}_k; \pmb{\phi}(s,0))).
\end{align}
The coupled iterates in~(\ref{eq:theta_update1}) and~(\ref{eq:lambda_update1}) form \qwhittleNFA as summarized in Algorithm \ref{Algorithm1}, which aims to learn the coupled parameters $(\pmb{\theta}^*, \lambda^*(s))$ such that $f(\pmb{\theta}^*, \pmb{\phi}(s,1)) =f(\pmb{\theta}^*, \pmb{\phi}(s,0)), \forall s\in\gS.$

\begin{table*}[t]
\caption{Comparison of settings in related works.}
\vspace{-0.1in}
\centering
\scalebox{0.9}{
\begin{tabular}{|c|c|c|c|c|}
 \hline
  \textbf{Algorithm} & \textbf{Noise } & \textbf{Approximation} & \textbf{Timescale} & \textbf{Whittle index}\\
 \hline\hline
\qwhittle \cite{avrachenkov2022whittle} & \textit{i.i.d.} & \ding{55} & \textbf{\textit{two-timescale}} & \ding{51}\\
\qwhittleLFA
\cite{xiong2023whittle} & \textit{i.i.d.} & linear & \textbf{\textit{two-timescale}} & \ding{51}\\
Q-Learning-LFA \cite{bhandari2018finite,melo2008analysis,zou2019finite} & \textbf{\textit{Markovian}} & linear & \textit{single-timescale} & \ding{55}\\
Q-Learning-NFA
\cite{cai2023neural,chen2019performance,fan2020theoretical,xu2020finite}
& \textbf{\textit{Markovian}} & \textbf{\textit{neural network}} & \textit{single-timescale}&\ding{55}\\
TD-Learning-LFA \cite{srikant2019finite} & \textbf{\textit{Markovian}} & linear & \textit{single-timescale} & \ding{55}\\
2TSA-IID
\cite{doan2020nonlinear,doan2019linear} & \textit{i.i.d.} & \ding{55} & \textbf{\textit{two-timescale}} & \ding{55}\\
2TSA-Markovian
\cite{doan2021finite} & \textbf{\textit{Markovian}} & \ding{55} & \textbf{\textit{two-timescale}} & \ding{55}\\
 \hline\hline
\qwhittleNFA ~{\textbf{(this work)}} & \textbf{\textit{Markovian}} & \textbf{\textit{neural network}} & \textbf{\textit{two-timescale}} & \ding{51}\\
\hline
\end{tabular}}
\label{tal:sotas}
\end{table*}

\begin{remark}
Unlike recent works for Q-learning with linear \cite{bhandari2018finite,melo2008analysis,zou2019finite} or neural network function approximations \cite{cai2023neural,fan2020theoretical,xu2020finite}, we do not assume an additional projection step of the updates of unknown parameters  $\pmb{\theta}_k$ in~(\ref{eq:theta_update1}) to confine $\pmb{\theta}_k, \forall k$ into a bounded set.  This projection step is often used to stabilize the iterates related to the unknown stationary distribution of the underlying Markov chain, which in some cases is impractical.  More recently, \cite{srikant2019finite} removed the extra projection step and established the finite-time convergence of TD learning, which is treated as a linear stochastic approximation algorithm. \cite{chen2019performance} extended it to the Q-learning with linear function approximation.  However, these state-of-the-art works only contained a single-timescale update on Q-function values, i.e., with the only unknown parameter $\pmb{\theta}$, while our \qwhittleNFA involves a two-timescale update between two coupled unknown parameters $\pmb{\theta}$ and $\lambda$ as in~(\ref{eq:theta_update1}) and~(\ref{eq:lambda_update1}). Our goal in this paper is to expand the frontier by providing a finite-time bound for \qwhittleNFA under Markovian noise without requiring an additional projection step.  {We summarize the differences between our work and existing literature in Table \ref{tal:sotas}.}
\end{remark}

\section{Finite-Time Analysis of \qwhittleNFA}\label{sec:finite-time-analysis}

In this section, we present the finite-time analysis of \qwhittleNFA for learning Whittle index $\lambda(s)$ of any state $s\in\gS$ when data are generated from a MDP.  To simplify notation, we abbreviate $\lambda(s)$ as $\lambda$ in the rest of the paper.  We start by first rewriting the updates of \qwhittleNFA in~(\ref{eq:theta_update1}) and~(\ref{eq:lambda_update1}) as a nonlinear two-timescale stochastic approximation (2TSA) in Section~\ref{sec:Q-Whittle-2TSA}.  

\subsection{A Nonlinear 2TSA Formulation with Neural Network Function}
\label{sec:Q-Whittle-2TSA}

We first show that \qwhittleNFA  can be rewritten as a variant of the nonlinear 2TSA.  For any fixed policy $\pi$, since the state of each arm $\{S_k\}$ evolves according to a Markov chain, we can construct a new variable $X_k=(S_k, A_k, S_{k+1})$, which also forms a Markov chain with { state space
    $\gX:=\{(s,a,s^\prime)|s\in\gS, \pi(a|s)\geq 0, p(s^\prime|s,a)>0\}.$}
Therefore, the coupled updates~(\ref{eq:theta_update1}) and~(\ref{eq:lambda_update1}) of \qwhittleNFA can be rewritten in the form of a nonlinear 2TSA \cite{doan2021finite}:
\begin{align}\label{eq:ge_lambda_2TSA}
    \pmb{\theta}_{k+1}=\pmb{\theta}_{k}+\alpha_kh(X_k,\pmb{\theta}_{k}, {\lambda}_k),
    \qquad {\lambda}_{k+1}={\lambda}_k+\eta_kg(X_k, \pmb{\theta}_k, \lambda_k),
\end{align}
where $\pmb{\theta}_0$ and ${\lambda}_0$ being arbitrarily initialized in $\mathbb{R}^{md}$ and $\mathbb{R}$, respectively; and $h(\cdot)$ and $g(\cdot)$ satisfy 
\begin{align}
    h(X_k,\pmb{\theta}_k, \lambda_k)&:=\nabla_{\pmb{\theta}}f(\pmb{\theta}_k;\pmb{\phi}(S_k, A_k))\Delta_k, \qquad\pmb{\theta}_k\in\mathbb{R}^{md}, \lambda_k\in\mathbb{R}, \label{eq:h_Q} \\
    g(X_k, \pmb{\theta}_k, \lambda_k)&:=f(\pmb{\theta}_k; \pmb{\phi}(s,1))-f(\pmb{\theta}_k; \pmb{\phi}(s,0)), \qquad\pmb{\theta}_k\in\mathbb{R}^{md}. \label{eq:g-Q}
\end{align}
Since $\eta_k\ll \alpha_k$, the dynamics of $\pmb{\theta}$ evolves much faster than those of ${\lambda}$.
We aim to establish the finite-time performance of the nonlinear 2TSA in (\ref{eq:ge_lambda_2TSA}), where $f(\cdot)$ is the neural network function defined in~(\ref{eq:neural_Q}).  This is equivalent to find the root\footnote{{The root   $(\pmb{\theta}^*, {\lambda}^*)$  of  the  nonlinear 2TSA (\ref{eq:ge_lambda_2TSA}) can be established by using the ODE method following the solution of suitably defined differential equations \cite{borkar2009stochastic,suttle2021reinforcement, avrachenkov2022whittle, doan2019linear,doan2020nonlinear, doan2021finite}, i.e.,} $\dot{\pmb{\theta}}=H(\pmb{\theta},{\lambda}),   \dot{{\lambda}}=\frac{\eta}{\alpha}G(\pmb{\theta}, {\lambda}),$
where a fixed stepsize is assumed for ease of expression at this moment.      } $(\pmb{\theta}^*, {\lambda}^*)$ of a system with \textit{two coupled} nonlinear equations $h:\gX\times\mathbb{R}^{md}\times\mathbb{R}\rightarrow\mathbb{R}^{md}$ and $g:\gX\times\mathbb{R}^{md}\times\mathbb{R}\rightarrow\mathbb{R}$ such that
\begin{align}\label{eq: 2TSA_EqPoint}
   H(\pmb{\theta}, {\lambda}):= \mathbb{E}_\mu[h(X, \pmb{\theta}, {\lambda})]=0, \qquad G(\pmb{\theta}, {\lambda}):= \mathbb{E}_\mu[g(X, \pmb{\theta}, {\lambda})]=0,
\end{align}
where $X$ is a random variable in finite state space $\gX$ with unknown distribution $\mu$.
For a fixed $\pmb{\theta}$, to study the stability of ${\lambda}$, we assume the condition on the existence of  a mapping such that ${\lambda}=y(\pmb{\theta})$ is the unique solution of $G(\pmb{\theta},{\lambda})=0.$ In particular, $y(\pmb{\theta})$ is given as
\begin{align}\label{eq:y_function}
    y(\pmb{\theta})= r(s,\!1)\!+\!\sum_{s^\prime}p(s^\prime |s,\!1)\max_a f(\pmb{\theta};\pmb{\phi}(s^\prime,a))\! -\!r(s,\!0)\!-\!\!\sum_{s^\prime}p(s^\prime |s,\!0)\max_a f(\pmb{\theta};\pmb{\phi}(s^\prime,a)).
\end{align}

\subsection{Main Results}\label{sec:results}

As inspired by \cite{doan2021finite}, the finite-time analysis of such a nonlinear 2TSA boils down to the choice of two step sizes $\{\alpha_k, \eta_k, \forall k\}$ and a Lyapunov function that couples the two iterates in~(\ref{eq:ge_lambda_2TSA}).  To this end, we first define the following two error terms:
\begin{align}\label{eq:residual}
    \tilde{\pmb{\theta}}_k&=\pmb{\theta}_k-{\pmb{\theta}}^*,\qquad
    \tilde{{\lambda}}_k={\lambda}_k-y(\pmb{\theta}_k),
\end{align}
which characterize the coupling between $\pmb{\theta}_k$ and ${\lambda}_k$.
If $\tilde{\pmb{\theta}}_k$ and $\tilde{\lambda}_k$ go to zero simultaneously, the convergence of $(\pmb{\theta}_k,{\lambda}_k)$ to $({\pmb{\theta}}^*,{\lambda}^*)$ can be established. Thus, to prove the convergence of $(\pmb{\theta}_k,{\lambda}_k)$ of the nonlinear 2TSA in (\ref{eq:ge_lambda_2TSA}) to its true value $({\pmb{\theta}}^*,{\lambda}^*)$, we can equivalently study the convergence of $(\tilde{\pmb{\theta}}_k,\tilde{{\lambda}}_k)$ by providing the finite-time analysis for the  mean squared error generated by (\ref{eq:ge_lambda_2TSA}). To couple the fast and slow iterates, we define the following weighted Lyapunov function
\begin{align}\label{eq:lyapunov-function}
    {M}(\pmb{\theta}_k, {\lambda}_k)&:=\frac{\eta_k}{\alpha_k}\|\tilde{\pmb{\theta}}_k\|^2+\|\tilde{{\lambda}}_k\|^2=\frac{\eta_k}{\alpha_k}\|\pmb{\theta}_k-{\pmb{\theta}}^*\|^2+\|{\lambda}_k-y(\pmb{\theta}_k)\|^2,
\end{align}
where $\|\cdot\|$ stands for the the Euclidean norm for vectors throughout the paper.  It is clear that the Lyapunov function ${M}(\pmb{\theta}_k, {\lambda}_k)$ combines the updates of $\pmb{\theta}$ and $\lambda$ with respect to the true neural network function $f(\pmb{\theta};\pmb{\phi}(s,a))$ in (\ref{eq:neural_Q}).

To this end, our goal turns to characterize finite-time convergence of $\mathbb{E}[{M}(\pmb{\theta}_k, {\lambda}_k)]$. However, it is challenging to directly finding the global optimum of the corresponding nonlinear equations due to the nonlinear parameterization of Q-function in \qwhittleNFA.  In addition, the  operators $h(\cdot), g(\cdot)$ and $y(\cdot)$ in (\ref{eq:h_Q}), (\ref{eq:g-Q}) and (\ref{eq:y_function}) directly relate with the convoluted neural network function $f(\pmb{\theta}; \pmb{\phi}(s,a))$ in (\ref{eq:neural_Q}), which hinders us to characterize the smoothness properties of theses operators. Such properties are often required for the analysis of stochastic approximation \cite{chen2019performance,doan2020nonlinear,doan2019linear}.

To mitigate this, \textbf{(Step 1)} we instead approximate the true neural network function $f(\pmb{\theta}, \pmb{\phi}(s,a))$ with a collection of local linearization $f_0(\pmb{\theta}; \pmb{\phi}(s,a))$ at the initial point $\pmb{\theta}_0$.  Based on the surrogate stationary point $\pmb{\theta}_0^*$ of $f_0(\pmb{\theta}; \pmb{\phi}(s,a))$, we correspondingly define a modified Lyapunov function $\hat{M}(\pmb{\theta}_k, {\lambda}_k)$ combining updates of $\pmb{\theta}$ and $\lambda$ with respect to such local linearization. Specifically, we have
\begin{align}\label{eq:lyapunov-function2}
    \hat{M}(\pmb{\theta}_k, {\lambda}_k)&:=\frac{\eta_k}{\alpha_k}\|\pmb{\theta}_k-\pmb{\theta}_0^*\|^2+\|{\lambda}_k-y_0(\pmb{\theta}_k)\|^2,
\end{align}
where $y_0(\cdot)$ is in the same expression as $y(\cdot)$ in~(\ref{eq:y_function}) by replacing $f(\cdot)$ with $f_0(\cdot)$, and we will describe this in details below.  \textbf{(Step 2)} We then study the convergence rate of the nonlinear 2TSA using this modified Lyapunov function under general conditions.  \textbf{(Step 3)} Finally, since the two coupled parameters $\pmb{\theta}$ and $\lambda$ in (\ref{eq:ge_lambda_2TSA}) are updated with respect to the true neural network function $f(\pmb{\theta};\pmb{\phi}(s,a))$ in (\ref{eq:neural_Q}) in \qwhittleNFA, while we characterize their convergence using the approximated neural network function in Step 2. Hence, this further requires us to characterize the approximation errors.  We visualize the above three steps in Figure~\ref{fig:Flowchart} and provide a proof sketch in Section~\ref{sec:proof-sketch}. Combing them together gives rise to our main theoretical results on the finite-time performance of \qwhittleNFA, which is formally stated in the following theorem. 

\begin{figure}
    \centering
    \includegraphics[width=0.98\textwidth]{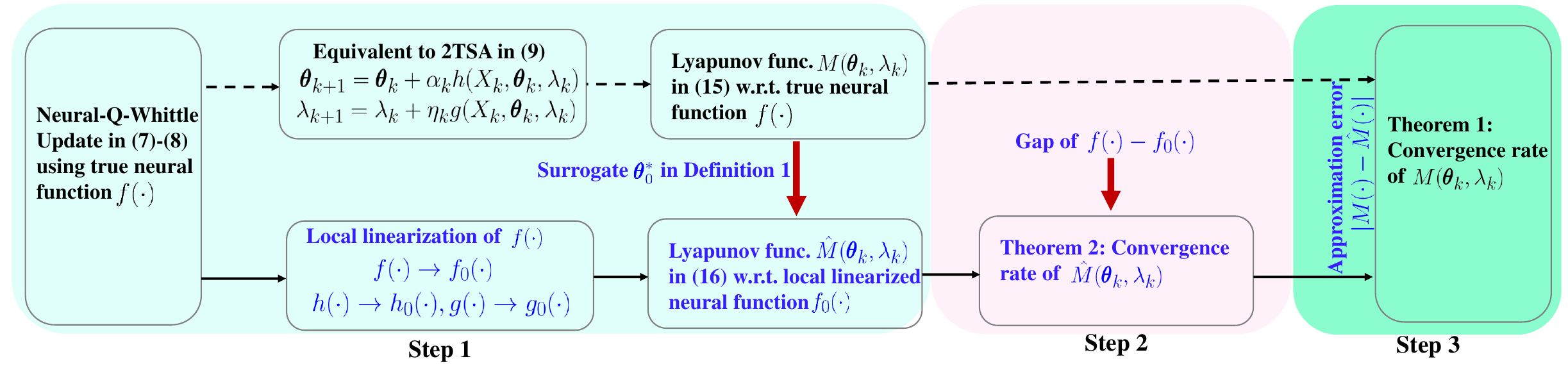}
    \caption{\qwhittleNFA operates w.r.t. true neural function $f(\cdot)$ with its finite-time performance given in Theorem~\ref{thm2:QW_convergence} (indicated in dashed lines). Our proofs operate in three steps: (i) Step 1: Obtain local linearization $f_0(\cdot)$ and define Lyapunov function $\hat{M}(\cdot)$ w.r.t. $f_0(\cdot)$. (ii) Step 2: Characterize the finite-time performance w.r.t. $\hat{M}(\cdot)$ using Lyapunov drift method. Since  \qwhittleNFA is updated w.r.t. $f(\cdot)$, we need to characterize the gap between $f(\cdot)$ and $f_0(\cdot)$. (iii) Step 3: Similarly, we characterize the approximation errors between $M(\cdot)$ and $\hat{M}(\cdot)$.}
    \label{fig:Flowchart}
\end{figure}

\begin{theorem}\label{thm2:QW_convergence}
Consider iterates $\{\pmb{\theta}_k\}$ and $\{{\lambda}_k\}$  generated by \qwhittleNFA  in~(\ref{eq:theta_update1}) and~(\ref{eq:lambda_update1}).  Given
   $\alpha_k=\frac{\alpha_0}{(k+1)},\eta_k=\frac{\eta_0}{(k+1)^{4/3}}$,
we have for $\forall k\geq \tau$
\begin{align}\label{eq:convergence_error}\nonumber
 & \mathbb{E}[{M}(\pmb{\theta}_{k+1},\lambda_{k+1})|\mathcal{F}_{k-\tau}]
   \leq \frac{2\tau^2\mathbb{E}[\hat{M}(\pmb{\theta}_\tau,\lambda_\tau)]}{(k+1)^2}+\frac{1200\alpha_0^3}{\eta_0}\frac{(C_1+\|\hat{\pmb{\theta}}_0\|)^2+(2C_1+\|\hat{\lambda}_0\|)^2}{(k+1)^{2/3}}\nonumber\allowdisplaybreaks\\
    +&\frac{2\eta_0c_0^2}{\alpha_0(1-\kappa)^2}\|span(\Pi_\mathcal{F}{f}({\pmb{\theta}^*})\!-\!{f}({\pmb{\theta}^*}))\|^2+\left(\frac{2}{(k+1)^{2/3}}\!+\!2\right)\mathcal{O}\Big(\frac{c_1^3(\|\pmb{\theta}_0\|\!+\!|\lambda_0|\!+\!1)^3} {m^{1/2}}\Big),
\end{align}
where $C_1:=c_1(\|\pmb{\theta}_0\|+\|\lambda_0\|+1)$ with $c_1$ being a proper chosen constant, $c_0$ is a constant defined in Assumption \ref{asump:span}, $\tau$ is the mixing time defined in~(\ref{def:mixing_time}), {$span$ denotes for the span semi-norm \cite{sharma2020approximate}}, and $\Pi_\gF$ represents the projection to the set of $\gF$ contianing all possible $f_0(\pmb{\theta};\pmb{\phi}(s,a))$ in (\ref{eq:neural_Q_linear}).
\end{theorem}
The first term on the right hand side~(\ref{eq:convergence_error}) corresponds to the bias compared to the Lyapunov function at the mixing time $\tau$, which goes to zero at a rate of $\mathcal{O}(1/k^2)$. The second term corresponds to the accumulated estimation error of the nonlinear 2TSA due to Markovian noise, which vanishes at the rate $\mathcal{O}(1/k^{2/3})$. Hence it dominates the overall convergence rate in~(\ref{eq:convergence_error}).  {The third term captures the distance between the optimal solution $(\pmb{\theta}^*, \lambda^*)$ to the true neural network function $f(\pmb{\theta}_k; \pmb{\phi}(s,a))$ in (\ref{eq:neural_Q}) and the optimal one $(\pmb{\theta}_0^*, y_0(\pmb{\theta}_0^*))$ with local linearization $f_0(\pmb{\theta}_k; \pmb{\phi}(s,a))$ in (\ref{eq:neural_Q_linear}), which quantifies the error when $f(\pmb{\theta}^*)$ does not fall into the function class $\mathcal{F}$.  The last term characterizes the distance between $f(\pmb{\theta}_k; \pmb{\phi}(s,a))$ and $f_0(\pmb{\theta}_k; \pmb{\phi}(s,a))$ with any $\pmb{\theta}_k$.}  Both terms diminish as $m\rightarrow\infty$. Theorem \ref{thm2:QW_convergence} implies the convergence to the  optimal value  $(\pmb{\theta}^*, \lambda^*)$ is bounded by the approximation error, which will diminish to zero as representation power of $f_0(\pmb{\theta}_k; \pmb{\phi}(s,a))$ increases when $m\rightarrow \infty.$ { Finally, we note that the right hand side~(\ref{eq:convergence_error}) ends up in $\mathcal{O}(1/k^2)+\mathcal{O}(1/k^{2/3})+c$, where $c$ is a constant and
 its value goes to $0$ as $m\rightarrow\infty$. This indicates the error bounds of linearization with the original neural network functions are controlled by the overparameterization value of $m$.}
Need to mention that a constant step size will result in a non-vanishing accumulated error as in \cite{chen2019performance}.



\begin{remark}
A finite-time analysis of nonlinear 2TSA was presented in \cite{mokkadem2006convergence}.  However, \cite{mokkadem2006convergence} required a stability condition that $\lim_{k\rightarrow\infty} (\pmb{\theta}_k,{\lambda}_k)=(\pmb{\theta}^*, {\lambda}^*)$, and both $h$ and $g$ are locally approximated as linear functions. \cite{doan2020nonlinear,xiong2023whittle} relaxed these conditions and provided a finite-time analysis under i.i.d. noise. These results were later extended to Markovian noise \cite{doan2021finite} under the assumption that $H$ function is strongly monotone in $\pmb{\theta}$ and $G$ function is strongly monotone in ${\lambda}$. Since \cite{doan2021finite} leveraged the techniques in \cite{doan2020nonlinear}, it needed to explicitly characterize the covariance between the error caused by Markovian noise and the parameters' residual error in (\ref{eq:residual}), leading to the convergence analysis much more intrinsic.    \cite{chen2019performance} exploited the mixing time to avoid the covariance between the error caused by Markovian noise and the parameters' residual error, however, it only considered the single timescale Q-learning with linear function approximation.
 Though our \qwhittleNFA can be rewritten as a nonlinear 2TSA, the nonlinear parameterization of Q-function caused by the neural network function approximation makes the aforementioned analysis not directly applicable to ours and requires additional characterization as highlighted in Figure~\ref{fig:Flowchart}. The explicit characterization of approximation errors further distinguish our work.
\end{remark}

\subsection{Proof Sketch}\label{sec:proof-sketch}
In this section, we sketch the proofs of the three steps shown in Figure~\ref{fig:Flowchart} as required for Theorem~\ref{thm2:QW_convergence}.

\subsubsection{Step 1: Approximated Solution of \qwhittleNFA}
We first approximate the optimal solution by projecting the Q-function in (\ref{eq:neural_Q}) to some function classes parameterized by $\pmb{\theta}$. The common choice of the projected function classes is the local linearization of $f(\pmb{\theta}; \pmb{\phi}(s,a))$ at the initial point $\pmb{\theta}_0$ \cite{cai2023neural,xu2020finite}, i.e., $\gF:=\{f_0(\pmb{\theta}; \pmb{\phi}(s,a)), \forall \pmb{\theta}\in{\Theta}\}$, where 
\begin{align}\label{eq:neural_Q_linear}
    f_0(\pmb{\theta}; \pmb{\phi}(s,a))=\frac{1}{\sqrt{m}}\sum_{r=1}^m b_r\mathds{1}\{\mathbf{w}_{r,0}^\intercal\pmb{\phi}(s,a) >0\}\mathbf{w}_r^\intercal\pmb{\phi}(s,a).
\end{align}
Then, we define the approximate stationary point $\pmb{\theta}_0^{*}$ with respect to $f_0(\pmb{\theta}; \pmb{\phi}(s,a))$ as follows.
\begin{defn}\label{def:stationary_point}[\cite{cai2023neural,xu2020finite}]
 A point   $\pmb{\theta}_0^*\in{\Theta}$ is said to be the approximate stationary point of Algorithm \ref{Algorithm1} if for all feasible $\pmb{\theta}\in\Theta$ it holds that
 $\mathbb{E}_{\mu,\pi,\mathcal{P}}[({\Delta}_{0}\cdot\nabla_{\pmb{\theta}}f_0(\pmb{\theta}; \pmb{\phi}(s,a)))^\intercal(\pmb{\theta}-\pmb{\theta}_0^*)]\geq 0, \forall \pmb{\theta}\in\Theta,$
{ with ${\Delta}_{0}:=[r(s,a)+(1-a)\lambda^*-I_0( \pmb{\theta})+\max_{a^\prime}f_0(\pmb{\theta}; \pmb{\phi}(s^\prime,a))-f_0(\pmb{\theta}; \pmb{\phi}(s,a))]$, where $I_0(\pmb{\theta})=\frac{1}{2S}\sum_{s\in\gS}[f_0(\pmb{\theta}; \pmb{\phi}(s,0))+f_0(\pmb{\theta}; \pmb{\phi}(s,1))]$.}
\end{defn}
Though there is a gap between the true neural function (\ref{eq:neural_Q}) and the approximated local linearized function (\ref{eq:neural_Q_linear}), the gap diminishes as the width of neural network i.e., $m$, becomes large \cite{cai2023neural, xu2020finite}.

With the approximated stationary point $\pmb{\theta}_0^*$,  we can redefine the two error terms in~(\ref{eq:residual}) as
\begin{align}\label{eq:residual-new}
    \hat{\pmb{\theta}}_k&=\pmb{\theta}_k-\pmb{\theta}_0^*,\qquad
    \hat{{\lambda}}_k={\lambda}_k-y_0(\pmb{\theta}_k),
\end{align}
using which we correspondingly define a modified Lyapunov function $\hat{M}(\pmb{\theta}_k, {\lambda}_k)$ in~(\ref{eq:lyapunov-function2}), where 
\begin{align}\label{eq:y0_function}
    {\hspace{-0.2cm}y_0(\pmb{\theta})\!=\!r(s,\!1)\!+\!\sum_{s^\prime}\!  p(s^\prime |s,1)\max_a \!f_0(\pmb{\theta};\pmb{\phi}(s^\prime,a))\!-\!r(s,\!0)\!-\!\!\!\sum_{s^\prime}p(s^\prime |s,0)}\max_a \!f_0(\pmb{\theta};\pmb{\phi}(s^\prime,a)).
\end{align}

\subsubsection{Step 2: Convergence Rate of $\hat{M}(\pmb{\theta}_k, \lambda_k)$ in (\ref{eq:lyapunov-function2}) }
Since we approximate the true neural network function $f(\pmb{\theta}; \pmb{\phi}(s,a))$ in (\ref{eq:neural_Q}) with the local linearized function $f_0(\pmb{\theta}; \pmb{\phi}(s,a))$ in (\ref{eq:neural_Q_linear}),
 the operators $h(\cdot)$ and $g(\cdot)$ in (\ref{eq:h_Q})-(\ref{eq:g-Q}) turn correspondingly to be
\begin{align} \label{eq:h0_Q}
    h_0(X_k,\pmb{\theta}_k, \lambda_k)&=\nabla_{\pmb{\theta}}f_0(\pmb{\theta}_k;\pmb{\phi}(S_k, A_k))\Delta_{k,0}, ~ g_0(\pmb{\theta}_k):=f_0(\pmb{\theta}_k; \pmb{\phi}(s,1))-f_0(\pmb{\theta}_k; \pmb{\phi}(s,0)),
\end{align}
{ with $\Delta_{k,0}:=r(S_k,A_k)+(1-A_k)\lambda_k-I_0( \pmb{\theta}_k)+\max_{a}f_0(\pmb{\theta}_k; \pmb{\phi}(S_{k+1},a))-f_0(\pmb{\theta}_k; \pmb{\phi}(S_{k},A_k))$.}

Before we present the finite-time error bound of the nonlinear 2TSA~(\ref{eq:ge_lambda_2TSA}) under Markovian noise, we first discuss the mixing time of the Markov chain $\{X_k\}$ and our assumptions. 
\begin{defn} [Mixing time \cite{chen2019performance}]
For any $\delta>0$, define $\tau_\delta$ as
\begin{align}\label{def:mixing_time}
    \tau_\delta=&\min\{k\geq 1: \|\mathbb{E}[h_0(X_k, \pmb{\theta}, {\lambda})|X_0=x]-H_0( \pmb{\theta}, {\lambda}) \|\leq \delta (\|\pmb{\theta}-\pmb{\theta}_0^*\|+\|{\lambda}-y_0(\pmb{\theta}_0^*)\|)\}.
\end{align}
\end{defn}

\begin{assumption}\label{assumption:markovian}
The Markov chain $\{X_k\}$ is irreducible and aperiodic.
Hence, there exists a unique stationary distribution $\mu$ \cite{levin2017markov},  and  constants $C>0$ and $\rho\in(0,1)$ such that
    $d_{TV}(P(X_k|X_0=x), \mu)\leq C\rho^k, \forall k\geq 0, x\in\gX,$
where $d_{TV}(\cdot,\cdot)$ is the total-variation (TV) distance \cite{levin2017markov}. 

\end{assumption}
\begin{remark}
Assumption \ref{assumption:markovian} is often assumed to study the asymptotic convergence of stochastic approximation under Markovian noise \cite{bertsekas1996neuro,borkar2009stochastic,chen2019performance}.
\end{remark}

\begin{lemma}\label{lem:h_lipschitz}
The function $h_0(X,\pmb{\theta}, \lambda)$ {defined in~(\ref{eq:h0_Q})} is globally Lipschitz continuous w.r.t $\pmb{\theta}$ and $\lambda$ uniformly in $X$, i.e., $\|h_0(X,\pmb{\theta}_1,\lambda_1)\!-\!h_0(X,\pmb{\theta}_2,\lambda_2)\|\leq L_{h,1}\|\pmb{\theta}_1-\pmb{\theta}_2\|+L_{h,2}\|\lambda_1-\lambda_2\|, \forall X\in\gX$, and $L_{h,1}=3, h_{h,2}=1$ are valid Lipschitz constants.
\end{lemma}

\begin{lemma}\label{lem:g_lipschitz}
 The function $g_0(\pmb{\theta})$ {defined in~(\ref{eq:h0_Q})} is linear and thus Lipschitz continuous in $\pmb{\theta}$, i.e., $\|g_0(\pmb{\theta}_1)-g_0(\pmb{\theta}_2)\|\leq L_g\|\pmb{\theta}_1-\pmb{\theta}_2\|$, and $L_g=2$ is a valid Lipschitz constant.
\end{lemma}

\begin{lemma}\label{lem:y_lipschitz}
 The function $y_0(\pmb{\theta})$ {defined in~(\ref{eq:y0_function})} is linear and thus Lipschitz continuous in ${\pmb{\theta}}$, i.e., $\|y_0(\pmb{\theta}_1)-y_0(\pmb{\theta}_2)\|\leq L_y\|\pmb{\theta}_1-\pmb{\theta}_2\|$, and $L_y=2$ is a valid Lipschitz constant.
\end{lemma}

\begin{remark}
{The Lipschitz continuity of $h_0$ guarantees the existence of a solution $\pmb{\theta}$ to the ODE  $\dot{\pmb{\theta}}$ for a fixed ${\lambda}$, while the  Lipschitz continuity of $g_0$ and $y_0$ ensures the existence of a solution ${\lambda}$ to the ODE $\dot{{\lambda}}$ when $\pmb{\theta}$ is fixed.   These lemmas often serve as assumptions when proving the convergence rate for both linear and nonlinear 2TSA \cite{konda2004convergence,mokkadem2006convergence,dalal2018finite,gupta2019finite,doan2020nonlinear,dalal2020tale,kaledin2020finite}. } 
\end{remark}


\begin{lemma}\label{lem:gh_monotone}
For a fixed ${\lambda}$, there exists a constant $\mu_1>0$ such that $h_0(X,\pmb{\theta}, \lambda)$ {defined in~(\ref{eq:h_Q})} satisfies
\begin{align}\nonumber
\mathbb{E}[\hat{\pmb{\theta}}^\intercal h_0(X,\pmb{\theta}, {\lambda})]\leq -\mu_1\|\hat{\pmb{\theta}}\|^2.
\end{align}
For fixed $\pmb{\theta}$, there exists a constant $\mu_2>0$ such that  $g_0(X,\pmb{\theta},\lambda)$ {defined in~(\ref{eq:g-Q})} satisfies
\begin{align}\nonumber
   \mathbb{E}[ \hat{{\lambda}} g_0(X,\pmb{\theta}, {\lambda})]\leq -\mu_2\|\hat{{\lambda}}\|^2.
\end{align}
\end{lemma}

\begin{remark}
Lemma~\ref{lem:gh_monotone} guarantees the stability and uniqueness of the solution $\pmb{\theta}$ to the ODE  $\dot{\pmb{\theta}}$ for a fixed ${\lambda}$, and the uniqueness of the solution ${\lambda}$ to the ODE $\dot{{\lambda}}$ for a fixed $\pmb{\theta}$.  This assumption can be viewed as a relaxation of the stronger monotone property of nonlinear mappings \cite{doan2020nonlinear, chen2019performance}, since it is automatically satisfied if $h$ and $g$ are strong monotone as assumed in \cite{doan2020nonlinear}.
\end{remark}


\begin{lemma}\label{lemma:mixing_time}
Under Assumption \ref{assumption:markovian} and Lemma \ref{lem:h_lipschitz}, there exist constants $C>0$, $\rho\in(0,1)$ and $L=\max(3, \max_X h_0(X, \pmb{\theta}_0^*),y_0(\pmb{\theta}_0^*))$ such that
\begin{align*}
     \tau_\delta\leq \frac{\log(1/\delta)+\log(2LCmd)}{\log(1/\rho)}.
\end{align*}
\end{lemma}

\begin{remark}
$\tau_\delta$ is equivalent to the mixing time of the underlying Markov chain 
satisfying $\lim_{\delta\rightarrow 0}\delta\tau_{\delta}=0$ \cite{chen2019performance}. For simplicity, we remove the subscript and denote it as $\tau$.
\end{remark}

We now present the finite-time error bound for the Lyapunov function $\hat{M}(\pmb{\theta}_k, \lambda_k)$ in (\ref{eq:lyapunov-function2}).

\begin{theorem}\label{thm:QW_convergence}
{ Consider iterates $\{\pmb{\theta}_k\}$ and $\{{\lambda}_k\}$  generated by \qwhittleNFA  in ~(\ref{eq:theta_update1}) and~(\ref{eq:lambda_update1})}.  Given Lemma \ref{lem:h_lipschitz}-\ref{lem:gh_monotone},
   $\alpha_k=\frac{\alpha_0}{(k+1)},\eta_k=\frac{\eta_0}{(k+1)^{4/3}}$, $C_1:=c_1(\|\pmb{\theta}_0\|+\|\lambda_0\|+1)$ with a constant $c_1$,
\begin{align}\label{eq:convergence_bound}
  \mathbb{E}[\hat{M}(\pmb{\theta}_{k+1},\lambda_{k+1})|\mathcal{F}_{k-\tau}]
   & \leq \frac{\tau^2\mathbb{E}[\hat{M}(\pmb{\theta}_\tau,\lambda_\tau)]}{(k+1)^2}+\frac{600\alpha_0^3}{\eta_0}\frac{(C_1+\|\hat{\pmb{\theta}}_0\|)^2+(2C_1+\|\hat{\lambda}_0\|)^2}{(k+1)^{2/3}}\nonumber\allowdisplaybreaks\\
    &\qquad\qquad\qquad+\frac{\mathcal{O}\Big(c_1^3(\|\pmb{\theta}_0\|\!+\!|\lambda_0|\!+\!1)^3 m^{-1/2}\Big)}{(k+1)^{2/3}}, \quad \forall k\geq \tau.
\end{align}
\end{theorem}

\subsubsection{Step 3: Approximation Error between $M(\pmb{\theta}_k, \lambda_k)$ and $\hat{M}(\pmb{\theta}_k, \lambda_k$)}\label{sec:final_results}

Finally, we characterize the approximation error between Lyapunov functions $M(\pmb{\theta}_k,\lambda_k)$ and $\hat{M}(\pmb{\theta}_k, \lambda_k)$.
Since we are dealing with long-term average MDP, we assume that the total variation of the MDP is bounded \cite{sharma2020approximate}.
\begin{assumption}
    There exists $0<\kappa <1$ such that
      $  \sup_{(s,a), (s^\prime, a^\prime)}\|p(\cdot|s,a)-p(\cdot|s^\prime, a^\prime)\|_{TV}=2\kappa$.
\end{assumption}
Hence, the Bellman operator is  a span-contraction operator \cite{sharma2020approximate}, i.e.,
\begin{align}\label{eq:span_contraction}
    span(\mathcal{T}f_0(\pmb{\theta}_0^*)-\mathcal{T}{f}({\pmb{\theta}^*}))\leq \kappa~ span(f_0(\pmb{\theta}_0^*)-{f}({\pmb{\theta}^*})).
\end{align}

\begin{assumption}\label{asump:span}

      $  \|{\pmb{\theta}}_0^*-\pmb{\theta}^*\|\leq c_0\|span(f_0(\pmb{\theta}_0^*)-{f}({\pmb{\theta}^*}))\|,$
    with $c_0$ being a positive constant.
\end{assumption}

\begin{lemma}\label{lemma:M-and-M_hat}
For ${M}(\pmb{\theta}_k, {\lambda}_k)$ in~(\ref{eq:lyapunov-function}) and $\hat{M}(\pmb{\theta}_k, {\lambda}_k)$ in~(\ref{eq:lyapunov-function2}), with constants $c_1$ and $c_0$ (Assumption \ref{asump:span}),
\begin{align*}
   {M}(\pmb{\theta}_k, {\lambda}_k)\leq 2 \hat{M}(\pmb{\theta}_k, {\lambda}_k)\!+\!\frac{2\eta_kc_0^2}{\alpha_k(1\!-\!\kappa)} \|span(\Pi_\mathcal{F}{f}({\pmb{\theta}^*})\!-\!{f}({\pmb{\theta}^*}))\|\!+\!2\mathcal{O}\Big(\frac{c_1^3(\|\pmb{\theta}_0\|\!+\!|\lambda_0|\!+\!1)^3}{ m^{1/2}}\Big).
\end{align*}
\end{lemma}



\section{Numerical Experiments}\label{sec:exp}
We numerically evaluate the performance of \qwhittleNFA using an example of circulant dynamics \cite{fu2019towards, avrachenkov2022whittle, biswas2021learn}.
The state space is $\gS=\{1, 2, 3, 4\}$. Rewards are $r(1,a)=-1, r(2,a)=r(3,a)=0,$ and $r(4,a)=1$ for $a\in\{0,1\}$. The dynamics of states are circulant and defined as
\begin{align*}
    P^1=\begin{bmatrix}
    0.5 & 0.5 & 0 & 0 \\
    0 & 0.5 & 0.5 & 0\\
    0 & 0 & 0.5 & 0.5\\
    0.5 & 0 & 0 & 0.5
    \end{bmatrix}~\text{and}~ P^0=\begin{bmatrix}
    0.5 & 0 & 0 & 0.5 \\
    0.5 & 0.5 & 0 & 0\\
    0 & 0.5 & 0.5 & 0\\
    0 & 0 & 0.5 & 0.5
    \end{bmatrix}.
\end{align*}
This indicates that the process either remains in its current state or increments if it is active (i.e., $a=1$), or it either remains the current state or decrements if it is passive (i.e., $a=0$).
The exact value of Whittle indices \cite{fu2019towards} are $\lambda(1)=-0.5, \lambda(2)=0.5, \lambda(3)=1,$ and $\lambda(4)= -1$.

\begin{wrapfigure}{rt}{0.6\linewidth}
    \centering
    \vspace{-0.2in}
\begin{minipage}{.3\textwidth}
 \centering
 \includegraphics[width=1\columnwidth]{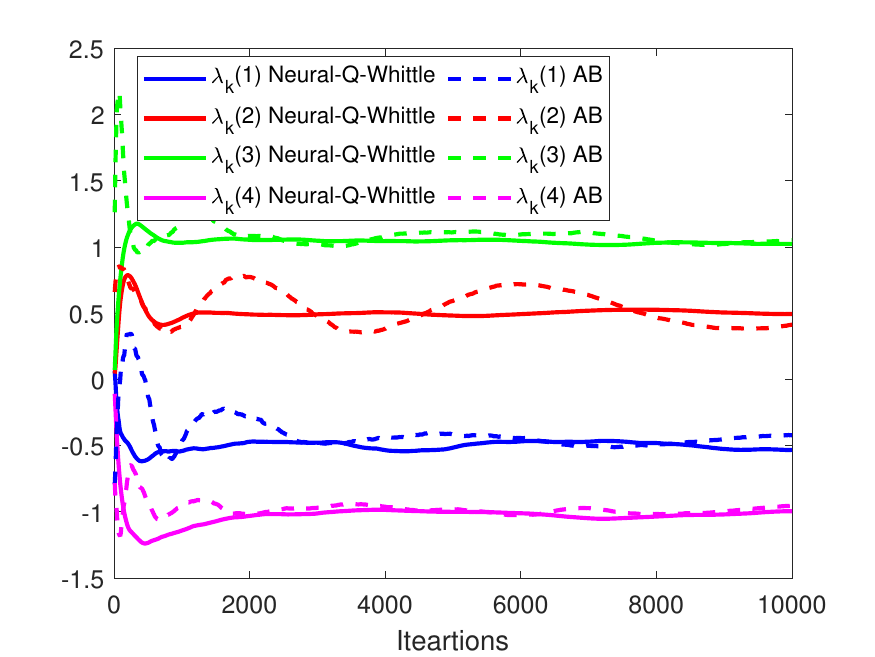}
 \vspace{-0.2in}
\subcaption{\qwhittleNFA vs.\\  \qwhittle \cite{avrachenkov2022whittle}.}
\label{fig:truewhittle}
 \end{minipage}\hfill
   \begin{minipage}{.3\textwidth}
 \centering
\includegraphics[width=1\columnwidth]{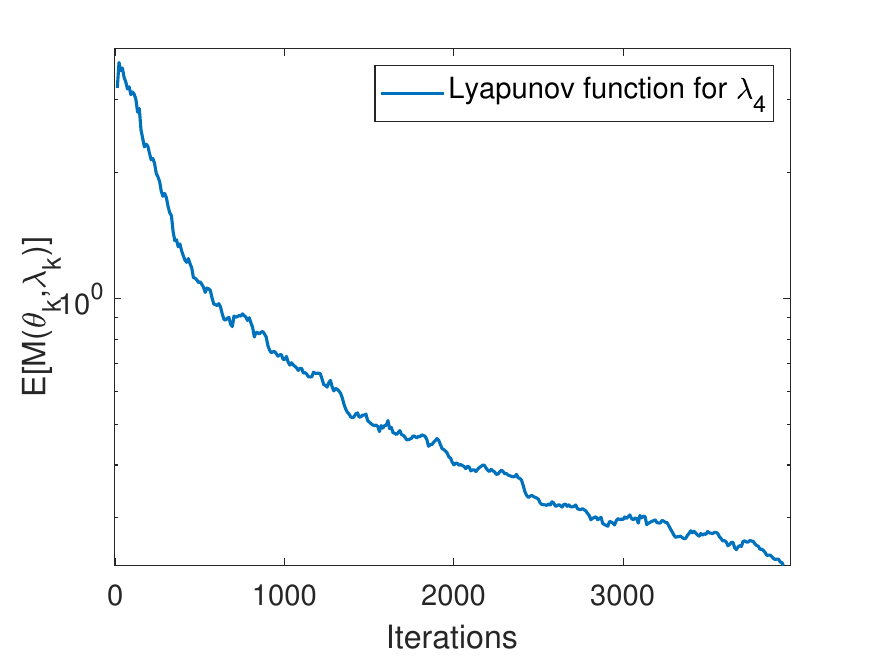}
 \vspace{-0.2in}
\subcaption{Convergence of Lyapunov function in (\ref{eq:lyapunov-function}).}
\label{fig:lyapunov}
 \end{minipage}
 \caption{Convergence of \qwhittleNFA.}
  \vspace{-0.2in}
\end{wrapfigure}

In our experiments, we set the learning rates as $\alpha_k=0.5/(k+1)$ and $\eta_k=0.1/(k+1)^{4/3}$. We use $\epsilon$-greedy for the exploration and exploitation tradeoff with $\epsilon=0.5$. We consider a two-layer neural network with the number of neurons in the hidden layer as $m=200.$ As described in Algorithm \ref{Algorithm1},  $b_r, \forall r$ are uniformly initialized in $\{-1, 1\}$ and $w_r, \forall r$ are initialized as a zero mean Gaussian distribution according to $\mathcal{N}(\pmb{0}, \mathbf{I}_d/d)$. These results are  carried out by Monte
Carlo simulations with 100 independent trials.

\textbf{Convergence to true Whittle index.}
First, we verify that \qwhittleNFA  convergences to true Whittle indices, and compare to \qwhittle, the first Whittle index based Q-learning algorithm.  As illustrated in Figure \ref{fig:truewhittle}, \qwhittleNFA guarantees the convergence to true Whittle indices and outperforms \qwhittle \cite{avrachenkov2022whittle} in the convergence speed. This is due to the fact that \qwhittleNFA updates the Whittle index of a specific state even when the current visited state is not that state.

\begin{wrapfigure}{rt}{0.6\linewidth}
    \centering
     \vspace{-0.2in}
\begin{minipage}{.3\textwidth}
 \centering
 \includegraphics[width=1\columnwidth]{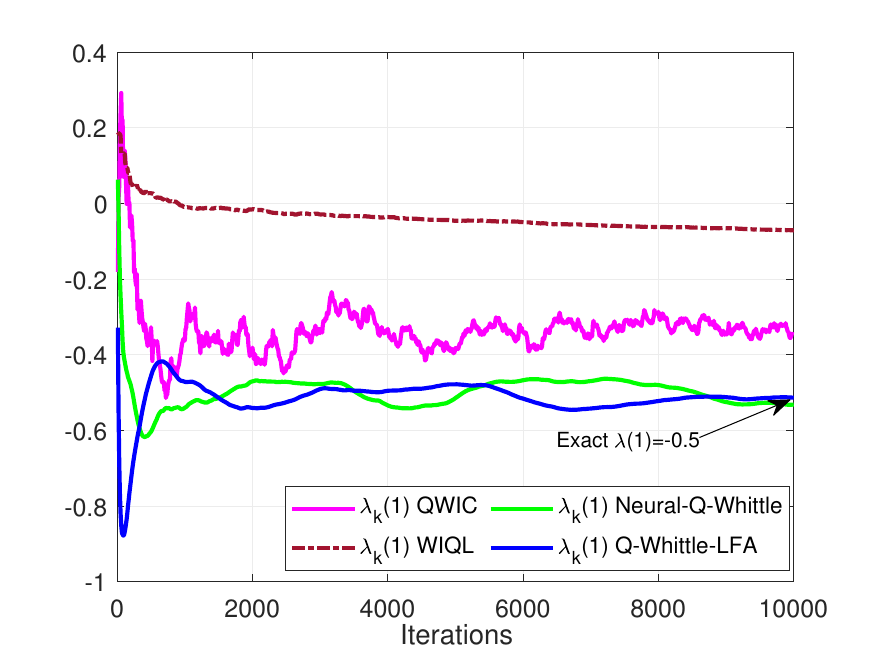}
 \vspace{-0.2in}
\subcaption{Whittle index $\lambda(1)$ for  $s=1$.}
\label{fig:s1}
 \end{minipage}\hfill
   \begin{minipage}{.3\textwidth}
 \centering
\includegraphics[width=1\columnwidth]{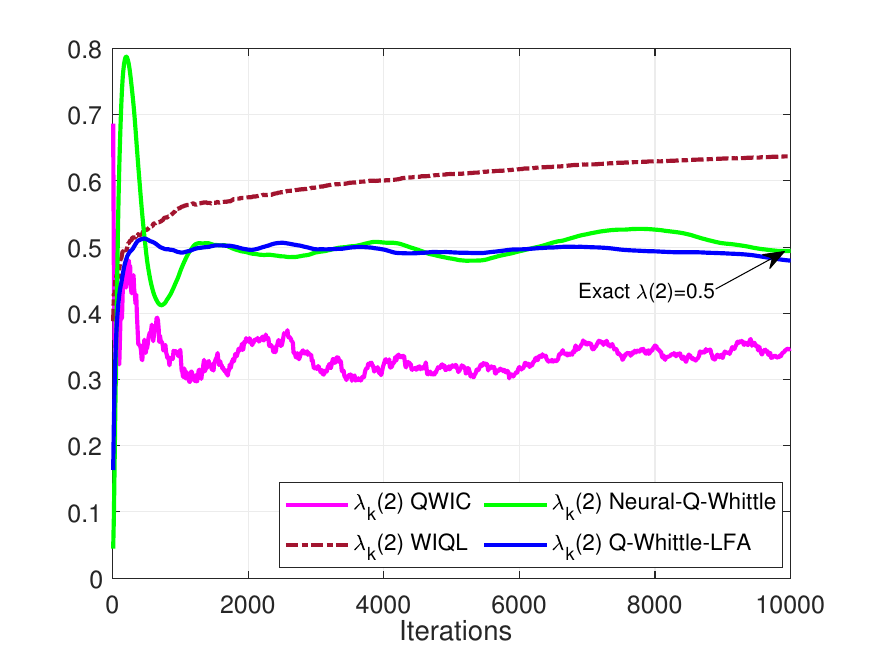}
 \vspace{-0.2in}
\subcaption{Whittle index $\lambda(2)$ for $s=2$.}
\label{fig:s2}
 \end{minipage}\hfill
 \begin{minipage}{.3\textwidth}
 \centering
 \includegraphics[width=1\columnwidth]{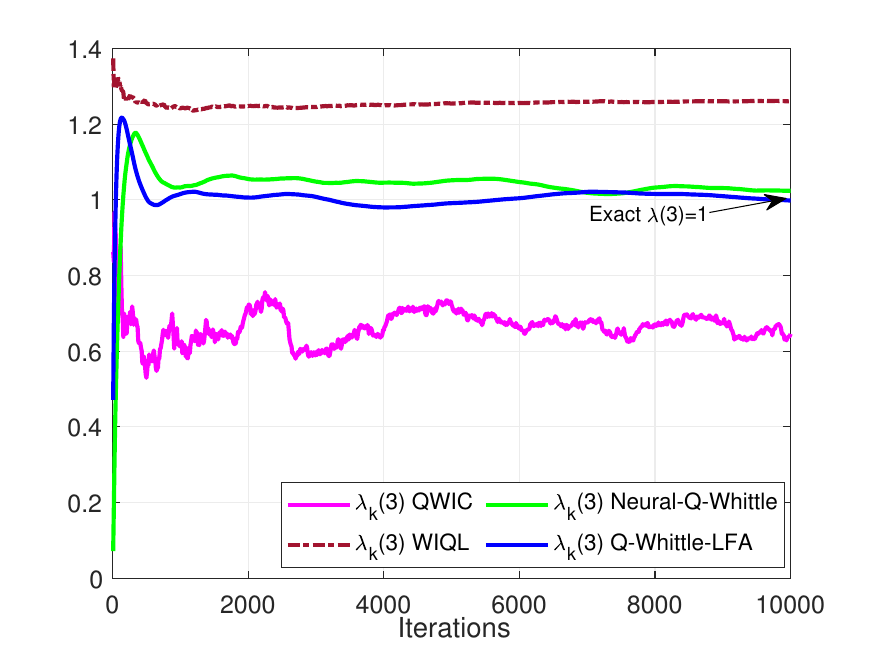}
 \vspace{-0.2in}
\subcaption{Whittle index $\lambda(3)$  $s=3$.}
\label{fig:s3}
 \end{minipage}\hfill
   \begin{minipage}{.3\textwidth}
 \centering
\includegraphics[width=1\columnwidth]{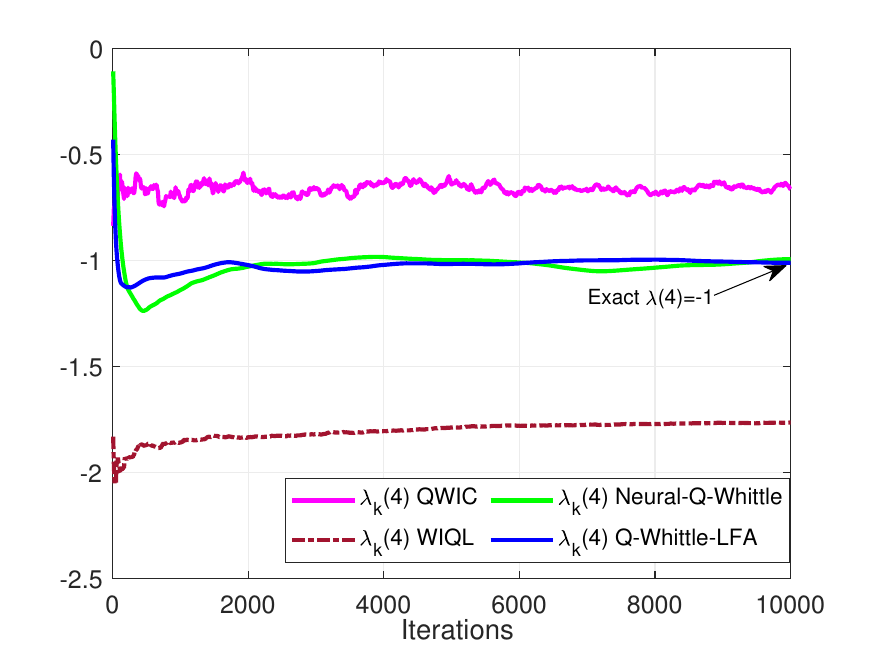}
 \vspace{-0.2in}
\subcaption{Whittle index $\lambda(4)$ for  $s=4$.}
\label{fig:s4}
 \end{minipage}

 \caption{Convergence comparison between \qwhittleNFA and benchmark algorithms.}
 \label{fig:figure4}
  \vspace{-0.1in}
\end{wrapfigure}

Second, we further compare with other other Whittle index learning algorithms, i.e., \qwhittleLFA \cite{xiong2023whittle}, WIQL \cite{biswas2021learn} and QWIC \cite{fu2019towards}in Figure \ref{fig:figure4}. As we observe from Figure \ref{fig:figure4}, only \qwhittleNFA and \qwhittleLFA in \cite{xiong2023whittle} can converge to the true Whittle indices for each state, while the other two benchmarks algorithms do not guarantee the convergence of true Whittle indices. Interestingly, the learning Whittle indices converge and maintain a correct relative order of magnitude, which is still be able to be used in real world problems \cite{xiong2023whittle}.
Moreover, we observe that \qwhittleNFA achieves similar convergence performance as \qwhittleLFA in the considered example, whereas the latter has been shown to achieve good performance in real world applications in \cite{xiong2023whittle}. Though this work focuses on the theoretical convergence analysis of Q-learning based whittle index under the neural network function approximation, it might be promising to implement it in real-world applications to fully leverage the strong representation ability of neural network functions, which serves as future investigation of this work.


\begin{wrapfigure}{rt}{0.35\linewidth}
    \centering
    \vspace{-0.2in}
 \includegraphics[width=0.35\columnwidth]{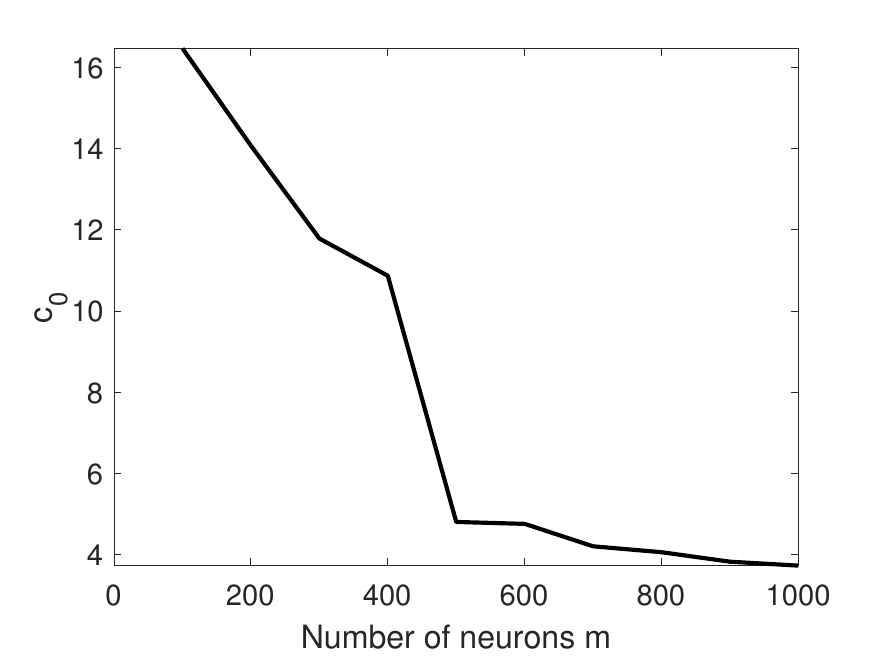}
 \vspace{-0.2in}
\caption{Verification of Assumption \ref{asump:span} w.r.t the constant $c_0$.}
 \label{fig:Assump3verify}
  \vspace{-0.15in}
\end{wrapfigure}

\textbf{Convergence of the Lyapunov function defined in (\ref{eq:lyapunov-function}).}
We also evaluate the convergence of the proposed Lyapunov function defined in (\ref{eq:lyapunov-function}), which is presented in Figure \ref{fig:lyapunov}. It depicts $\mathbb{E}[M(\pmb{\theta}_k,\lambda_k)]$ vs. the number of iterations in logarithmic scale.  For ease of presentation, we only take  state $s=4$ as an illustrative example. It is clear that $M(\pmb{\theta}_k,\lambda_k)$ converges to zero as the number of iterations increases, which is in alignment with our theoretical results in Theorem \ref{thm2:QW_convergence}.

\textbf{Verification of Assumption \ref{asump:span}.}
We now verify Assumption \ref{asump:span} that the gap between $\pmb{\theta}_0^*$ and $\pmb{\theta}^*$ can be bounded by the span of $f_0(\pmb{\theta}_0^*)$ and $f(\pmb{\theta}^*)$ with a constant $c_0$. In Figure \ref{fig:Assump3verify}, we show $c_0$ as a function of the number of neurons in the hidden layer $m$. It clearly indicates that  constant $c_0$ exists and decreases as the number of neurons grows larger.


\section{Conclusion}

We presented \qwhittleNFA, a Whittle index based Q-learning algorithm for RMAB with neural network function approximation.  We proved that \qwhittleNFA achieves an $\mathcal{O}(1/k^{2/3})$ convergence rate, where $k$ is the number of iterations when data are generated from a Markov chain and   Q-function is approximated by a ReLU neural network. By viewing \qwhittleNFA as 2TSA and leveraging the Lyapunov drift method, we removed the projection step on parameter update of Q-learning with neural network function approximation.  Extending the current framework to two-timescale Q-learning (i.e., the coupled iterates between Q-function values and Whittle indices) with general deep neural network approximation is our future work. 

\newpage
\section*{Acknowledgements}
This work was supported in part by the National Science Foundation (NSF) grant RINGS-2148309, and was supported in part by funds from OUSD R\&E, NIST, and industry partners as specified in the Resilient \& Intelligent NextG Systems (RINGS) program. This work was also supported in part by the U.S. Army Research Office (ARO) grant W911NF-23-1-0072, and the U.S. Department of Energy (DOE) grant DE-EE0009341. Any opinions, findings, and conclusions or recommendations expressed in this material are those of the authors and do not necessarily reflect the views of the funding agencies. 

\bibliography{refs,refs2}

\begin{thebibliography}{10}

\bibitem{abounadi2001learning}
Jinane Abounadi, Dimitrib Bertsekas, and Vivek~S Borkar.
\newblock Learning algorithms for markov decision processes with average cost.
\newblock {\em SIAM Journal on Control and Optimization}, 40(3):681--698, 2001.

\bibitem{achiam2019towards}
Joshua Achiam, Ethan Knight, and Pieter Abbeel.
\newblock Towards characterizing divergence in deep q-learning.
\newblock {\em arXiv preprint arXiv:1903.08894}, 2019.

\bibitem{avrachenkov2022whittle}
Konstantin~E Avrachenkov and Vivek~S Borkar.
\newblock Whittle index based q-learning for restless bandits with average
  reward.
\newblock {\em Automatica}, 139:110186, 2022.

\bibitem{bertsekas1996neuro}
Dimitri Bertsekas and John~N Tsitsiklis.
\newblock {\em Neuro-dynamic programming}.
\newblock Athena Scientific, 1996.

\bibitem{bertsimas2000restless}
Dimitris Bertsimas and Jos{\'e} Ni{\~n}o-Mora.
\newblock {Restless Bandits, Linear Programming Relaxations, and A Primal-Dual
  Index Heuristic}.
\newblock {\em Operations Research}, 48(1):80--90, 2000.

\bibitem{bhandari2018finite}
Jalaj Bhandari, Daniel Russo, and Raghav Singal.
\newblock A finite time analysis of temporal difference learning with linear
  function approximation.
\newblock In {\em Conference on learning theory}, pages 1691--1692. PMLR, 2018.

\bibitem{bhatnagar2009natural}
Shalabh Bhatnagar, Richard~S Sutton, Mohammad Ghavamzadeh, and Mark Lee.
\newblock {Natural Actor--Critic Algorithms}.
\newblock {\em Automatica}, 45(11):2471--2482, 2009.

\bibitem{biswas2021learn}
Arpita Biswas, Gaurav Aggarwal, Pradeep Varakantham, and Milind Tambe.
\newblock Learn to intervene: An adaptive learning policy for restless bandits
  in application to preventive healthcare.
\newblock In {\em Proc. of IJCAI}, 2021.

\bibitem{borkar2009stochastic}
Vivek~S Borkar.
\newblock {\em {Stochastic Approximation: A Dynamical Systems Viewpoint}},
  volume~48.
\newblock Springer, 2009.

\bibitem{borkar2018reinforcement}
Vivek~S Borkar and Karan Chadha.
\newblock A reinforcement learning algorithm for restless bandits.
\newblock In {\em 2018 Indian Control Conference (ICC)}, pages 89--94. IEEE,
  2018.

\bibitem{borkar1997actor}
Vivek~S Borkar and Vijaymohan~R Konda.
\newblock {The Actor-Critic Algorithm as Multi-Time-Scale Stochastic
  Approximation}.
\newblock {\em Sadhana}, 22(4):525--543, 1997.

\bibitem{borkar2017index}
Vivek~S Borkar, K~Ravikumar, and Krishnakant Saboo.
\newblock An index policy for dynamic pricing in cloud computing under price
  commitments.
\newblock {\em Applicationes Mathematicae}, 44:215--245, 2017.

\bibitem{cai2023neural}
Qi~Cai, Zhuoran Yang, Jason~D Lee, and Zhaoran Wang.
\newblock Neural temporal difference and q learning provably converge to global
  optima.
\newblock {\em Mathematics of Operations Research}, 2023.

\bibitem{chen2020finite}
Zaiwei Chen, Siva Theja~Maguluri, Sanjay Shakkottai, and Karthikeyan Shanmugam.
\newblock {Finite-Sample Analysis of Stochastic Approximation Using Smooth
  Convex Envelopes}.
\newblock {\em arXiv e-prints}, pages arXiv--2002, 2020.

\bibitem{chen2019performance}
Zaiwei Chen, Sheng Zhang, Thinh~T Doan, Siva~Theja Maguluri, and John-Paul
  Clarke.
\newblock {Performance of Q-learning with Linear Function Approximation:
  Stability and Finite-Time Analysis}.
\newblock {\em arXiv preprint arXiv:1905.11425}, 2019.

\bibitem{dai2011non}
Wenhan Dai, Yi~Gai, Bhaskar Krishnamachari, and Qing Zhao.
\newblock {The Non-Bayesian Restless Multi-Armed Bandit: A Case of
  Near-Logarithmic Regret}.
\newblock In {\em Proc. of IEEE ICASSP}, 2011.

\bibitem{dalal2020tale}
Gal Dalal, Balazs Szorenyi, and Gugan Thoppe.
\newblock A tale of two-timescale reinforcement learning with the tightest
  finite-time bound.
\newblock In {\em Proceedings of the AAAI Conference on Artificial
  Intelligence}, volume~34, pages 3701--3708, 2020.

\bibitem{dalal2018finite}
Gal Dalal, Gugan Thoppe, Bal{\'a}zs Sz{\"o}r{\'e}nyi, and Shie Mannor.
\newblock Finite sample analysis of two-timescale stochastic approximation with
  applications to reinforcement learning.
\newblock In {\em Conference On Learning Theory}, pages 1199--1233. PMLR, 2018.

\bibitem{doan2020nonlinear}
Thinh~T Doan.
\newblock Nonlinear two-time-scale stochastic approximation: Convergence and
  finite-time performance.
\newblock {\em arXiv preprint arXiv:2011.01868}, 2020.

\bibitem{doan2021finite}
Thinh~T Doan.
\newblock Finite-time convergence rates of nonlinear two-time-scale stochastic
  approximation under markovian noise.
\newblock {\em arXiv preprint arXiv:2104.01627}, 2021.

\bibitem{doan2019linear}
Thinh~T Doan and Justin Romberg.
\newblock {Linear Two-Time-Scale Stochastic Approximation A Finite-Time
  Analysis}.
\newblock In {\em Proc. of Allerton}, 2019.

\bibitem{fan2020theoretical}
Jianqing Fan, Zhaoran Wang, Yuchen Xie, and Zhuoran Yang.
\newblock A theoretical analysis of deep q-learning.
\newblock In {\em Learning for Dynamics and Control}, pages 486--489. PMLR,
  2020.

\bibitem{fu2019towards}
Jing Fu, Yoni Nazarathy, Sarat Moka, and Peter~G Taylor.
\newblock Towards q-learning the whittle index for restless bandits.
\newblock In {\em 2019 Australian \& New Zealand Control Conference (ANZCC)},
  pages 249--254. IEEE, 2019.

\bibitem{gupta2019finite}
Harsh Gupta, Rayadurgam Srikant, and Lei Ying.
\newblock Finite-time performance bounds and adaptive learning rate selection
  for two time-scale reinforcement learning.
\newblock {\em Advances in Neural Information Processing Systems}, 32, 2019.

\bibitem{jiang2023online}
Bowen Jiang, Bo~Jiang, Jian Li, Tao Lin, Xinbing Wang, and Chenghu Zhou.
\newblock Online restless bandits with unobserved states.
\newblock In {\em Proc. of ICML}, 2023.

\bibitem{jung2019regret}
Young~Hun Jung and Ambuj Tewari.
\newblock {Regret Bounds for Thompson Sampling in Episodic Restless Bandit
  Problems}.
\newblock {\em Proc. of NeurIPS}, 2019.

\bibitem{kaledin2020finite}
Maxim Kaledin, Eric Moulines, Alexey Naumov, Vladislav Tadic, and Hoi-To Wai.
\newblock Finite time analysis of linear two-timescale stochastic approximation
  with markovian noise.
\newblock In {\em Conference on Learning Theory}, pages 2144--2203. PMLR, 2020.

\bibitem{killian2021q}
Jackson~A Killian, Arpita Biswas, Sanket Shah, and Milind Tambe.
\newblock Q-learning lagrange policies for multi-action restless bandits.
\newblock In {\em Proceedings of the 27th ACM SIGKDD Conference on Knowledge
  Discovery \& Data Mining}, pages 871--881, 2021.

\bibitem{killian2021beyond}
Jackson~A Killian, Andrew Perrault, and Milind Tambe.
\newblock {Beyond" To Act or Not to Act": Fast Lagrangian Approaches to General
  Multi-Action Restless Bandits}.
\newblock In {\em Proc.of AAMAS}, 2021.

\bibitem{konda2000actor}
Vijay~R Konda and John~N Tsitsiklis.
\newblock {Actor-Critic Algorithms}.
\newblock In {\em Proc. of NIPS}, 2000.

\bibitem{konda2004convergence}
Vijay~R Konda and John~N Tsitsiklis.
\newblock Convergence rate of linear two-time-scale stochastic approximation.
\newblock {\em The Annals of Applied Probability}, 14(2):796--819, 2004.

\bibitem{levin2017markov}
David~A Levin and Yuval Peres.
\newblock {\em Markov chains and mixing times}, volume 107.
\newblock American Mathematical Soc., 2017.

\bibitem{liu2011logarithmic}
Haoyang Liu, Keqin Liu, and Qing Zhao.
\newblock {Logarithmic Weak Regret of Non-Bayesian Restless Multi-Armed
  Bandit}.
\newblock In {\em Proc of IEEE ICASSP}, 2011.

\bibitem{liu2012learning}
Haoyang Liu, Keqin Liu, and Qing Zhao.
\newblock {Learning in A Changing World: Restless Multi-Armed Bandit with
  Unknown Dynamics}.
\newblock {\em IEEE Transactions on Information Theory}, 59(3):1902--1916,
  2012.

\bibitem{mate2020collapsing}
Aditya Mate, Jackson Killian, Haifeng Xu, Andrew Perrault, and Milind Tambe.
\newblock Collapsing bandits and their application to public health
  intervention.
\newblock {\em Advances in Neural Information Processing Systems},
  33:15639--15650, 2020.

\bibitem{mate2021risk}
Aditya Mate, Andrew Perrault, and Milind Tambe.
\newblock {Risk-Aware Interventions in Public Health: Planning with Restless
  Multi-Armed Bandits}.
\newblock In {\em Proc.of AAMAS}, 2021.

\bibitem{melo2008analysis}
Francisco~S Melo, Sean~P Meyn, and M~Isabel Ribeiro.
\newblock An analysis of reinforcement learning with function approximation.
\newblock In {\em Proceedings of the 25th international conference on Machine
  learning}, pages 664--671, 2008.

\bibitem{meshram2016optimal}
Rahul Meshram, Aditya Gopalan, and D~Manjunath.
\newblock Optimal recommendation to users that react: Online learning for a
  class of pomdps.
\newblock In {\em 2016 IEEE 55th Conference on Decision and Control (CDC)},
  pages 7210--7215. IEEE, 2016.

\bibitem{mokkadem2006convergence}
Abdelkader Mokkadem and Mariane Pelletier.
\newblock Convergence rate and averaging of nonlinear two-time-scale stochastic
  approximation algorithms.
\newblock {\em The Annals of Applied Probability}, 16(3):1671--1702, 2006.

\bibitem{nakhleh2021neurwin}
Khaled Nakhleh, Santosh Ganji, Ping-Chun Hsieh, I~Hou, Srinivas Shakkottai,
  et~al.
\newblock Neurwin: Neural whittle index network for restless bandits via deep
  rl.
\newblock {\em Advances in Neural Information Processing Systems}, 34, 2021.

\bibitem{nakhleh2022deeptop}
Khaled Nakhleh, I~Hou, et~al.
\newblock Deeptop: Deep threshold-optimal policy for mdps and rmabs.
\newblock {\em arXiv preprint arXiv:2209.08646}, 2022.

\bibitem{ortner2012regret}
Ronald Ortner, Daniil Ryabko, Peter Auer, and R{\'e}mi Munos.
\newblock {Regret Bounds for Restless Markov Bandits}.
\newblock In {\em Proc. of Algorithmic Learning Theory}, 2012.

\bibitem{pagare2023full}
Tejas Pagare, Vivek Borkar, and Konstantin Avrachenkov.
\newblock Full gradient deep reinforcement learning for average-reward
  criterion.
\newblock {\em arXiv preprint arXiv:2304.03729}, 2023.

\bibitem{papadimitriou1994complexity}
Christos~H Papadimitriou and John~N Tsitsiklis.
\newblock {The Complexity of Optimal Queueing Network Control}.
\newblock In {\em Proc. of IEEE Conference on Structure in Complexity Theory},
  1994.

\bibitem{puterman1994markov}
Martin~L Puterman.
\newblock {\em {Markov Decision Processes: Discrete Stochastic Dynamic
  Programming}}.
\newblock John Wiley \& Sons, 1994.

\bibitem{qu2020finite}
Guannan Qu and Adam Wierman.
\newblock {Finite-Time Analysis of Asynchronous Stochastic Approximation and $
  Q $-Learning}.
\newblock In {\em Proc. of COLT}, 2020.

\bibitem{sharma2020approximate}
Hiteshi Sharma, Mehdi Jafarnia-Jahromi, and Rahul Jain.
\newblock Approximate relative value learning for average-reward continuous
  state mdps.
\newblock In {\em Uncertainty in Artificial Intelligence}, pages 956--964.
  PMLR, 2020.

\bibitem{srikant2019finite}
Rayadurgam Srikant and Lei Ying.
\newblock Finite-time error bounds for linear stochastic approximation andtd
  learning.
\newblock In {\em Conference on Learning Theory}, pages 2803--2830. PMLR, 2019.

\bibitem{suttle2021reinforcement}
Wesley Suttle, Kaiqing Zhang, Zhuoran Yang, Ji~Liu, and David Kraemer.
\newblock {Reinforcement Learning for Cost-Aware Markov Decision Processes}.
\newblock In {\em Proc. of ICML}, 2021.

\bibitem{tekin2012online}
Cem Tekin and Mingyan Liu.
\newblock {Online Learning of Rested and Restless Bandits}.
\newblock {\em IEEE Transactions on Information Theory}, 58(8):5588--5611,
  2012.

\bibitem{tsitsiklis1999average}
John~N Tsitsiklis and Benjamin Van~Roy.
\newblock {Average Cost Temporal-Difference Learning}.
\newblock {\em Automatica}, 35(11):1799--1808, 1999.

\bibitem{wan2021learning}
Yi~Wan, Abhishek Naik, and Richard~S Sutton.
\newblock {Learning and Planning in Average-Reward Markov Decision Processes}.
\newblock In {\em Proc. of ICML}, 2021.

\bibitem{wang2020restless}
Siwei Wang, Longbo Huang, and John Lui.
\newblock {Restless-UCB, an Efficient and Low-complexity Algorithm for Online
  Restless Bandits}.
\newblock In {\em Proc. of NeurIPS}, 2020.

\bibitem{weber1990index}
Richard~R Weber and Gideon Weiss.
\newblock {On An Index Policy for Restless Bandits}.
\newblock {\em Journal of applied probability}, pages 637--648, 1990.

\bibitem{wei2020model}
Chen-Yu Wei, Mehdi~Jafarnia Jahromi, Haipeng Luo, Hiteshi Sharma, and Rahul
  Jain.
\newblock {Model-free Reinforcement Learning in Infinite-Horizon Average-Reward
  Markov Decision Processes}.
\newblock In {\em Proc. of ICML}, 2020.

\bibitem{whittle1988restless}
Peter Whittle.
\newblock {Restless Bandits: Activity Allocation in A Changing World}.
\newblock {\em Journal of applied probability}, pages 287--298, 1988.

\bibitem{xiong2022reinforcement}
Guojun Xiong, Jian Li, and Rahul Singh.
\newblock {Reinforcement Learning Augmented Asymptotically Optimal Index
  Policies for Finite-Horizon Restless Bandits}.
\newblock In {\em Proc. of AAAI}, 2022.

\bibitem{xiong2022indexwireless}
Guojun Xiong, Xudong Qin, Bin Li, Rahul Singh, and Jian Li.
\newblock {Index-aware Reinforcement Learning for Adaptive Video Streaming at
  the Wireless Edge}.
\newblock In {\em Proc. of ACM MobiHoc}, 2022.

\bibitem{xiong2022Nips}
Guojun Xiong, Shufan Wang, and Jian Li.
\newblock Learning infinite-horizon average-reward restless multi-action
  bandits via index awareness.
\newblock {\em Proc. of NeurIPS}, 2022.

\bibitem{xiong2023whittle}
Guojun Xiong, Shufan Wang, Jian Li, and Rahul Singh.
\newblock Whittle index based q-learning for wireless edge caching with linear
  function approximation.
\newblock {\em arXiv preprint arXiv:2202.13187}, 2022.

\bibitem{xiong2022reinforcementcache}
Guojun Xiong, Shufan Wang, Gang Yan, and Jian Li.
\newblock {Reinforcement Learning for Dynamic Dimensioning of Cloud Caches: A
  Restless Bandit Approach}.
\newblock In {\em Proc. of IEEE INFOCOM}, 2022.

\bibitem{xu2020finite}
Pan Xu and Quanquan Gu.
\newblock A finite-time analysis of q-learning with neural network function
  approximation.
\newblock In {\em International Conference on Machine Learning}, pages
  10555--10565. PMLR, 2020.

\bibitem{yang2019provably}
Zhuoran Yang, Yongxin Chen, Mingyi Hong, and Zhaoran Wang.
\newblock {Provably Global Convergence of Actor-Critic: A Case for Linear
  Quadratic Regulator with Ergodic Cost}.
\newblock In {\em Proc. of NeurIPS}, 2019.

\bibitem{yu2018deadline}
Zhe Yu, Yunjian Xu, and Lang Tong.
\newblock {Deadline Scheduling as Restless Bandits}.
\newblock {\em IEEE Transactions on Automatic Control}, 63(8):2343--2358, 2018.

\bibitem{zhang2021average}
Shangtong Zhang, Yi~Wan, Richard~S Sutton, and Shimon Whiteson.
\newblock {Average-Reward Off-Policy Policy Evaluation with Function
  Approximation}.
\newblock {\em arXiv preprint arXiv:2101.02808}, 2021.

\bibitem{zhang2021finite}
Sheng Zhang, Zhe Zhang, and Siva~Theja Maguluri.
\newblock {Finite Sample Analysis of Average-Reward TD Learning and $ Q
  $-Learning}.
\newblock {\em Proc. of NeurIPS}, 2021.

\bibitem{zou2019finite}
Shaofeng Zou, Tengyu Xu, and Yingbin Liang.
\newblock Finite-sample analysis for sarsa with linear function approximation.
\newblock {\em Advances in neural information processing systems}, 32, 2019.

\end{thebibliography}
\bibliographystyle{plain}

\clearpage
\appendix
\section{Related Work}\label{sec:related}

\textbf{Online Restless Bandits.} The online RMAB setting, where the underlying MDPs are unknown, has been gaining attention, e.g., \cite{dai2011non,liu2011logarithmic,liu2012learning,tekin2012online,ortner2012regret,jung2019regret}.  However, these methods do not exploit the special structure available in the problem and contend directly with an extremely high dimensional state-action space yielding the algorithms to be too slow to be useful.  Recently,  RL based algorithms have been developed \cite{borkar2018reinforcement,fu2019towards,wang2020restless,biswas2021learn,killian2021q,xiong2022reinforcement,xiong2022Nips,xiong2022reinforcementcache,xiong2022indexwireless,avrachenkov2022whittle},
to explore the problem structure through index policies.
For instance, \cite{fu2019towards} proposed a Q-learning algorithm for Whittle index under the discounted setting, which lacks of convergence guarantees.  \cite{biswas2021learn} approximated Whittle index using the difference of $Q(s,1)-Q(s,0)$ for any state $s$, which is not guaranteed to converge to the true Whittle index in general scenarios.  To our best knowledge, the \qwhittle in (\ref{eq: Convention_Q_update})-(\ref{eq:convention_lambda_update}) proposed by  \cite{avrachenkov2022whittle} is the first algorithm with a rigorous asymptotic analysis.
Therefore, \cite{fu2019towards,avrachenkov2022whittle,biswas2021learn,killian2021q} lacked finite-time performance analysis and multi-timescale stochastic approximation algorithms usually suffer from slow convergence.

\cite{wang2020restless,xiong2022reinforcement} designed model-based low-complexity policy but is constrained to either a specific Markovian model or depends on a simulator for a finite-horizon setting which cannot be directly applied here.
Latter on, \cite{xiong2023whittle} showed the finite-time convergence performance under the \qwhittle setting of \cite{avrachenkov2022whittle} with linear function approximation. However, the underlying assumption in \cite{avrachenkov2022whittle,xiong2023whittle}  is that data samples are drawn i.i.d per iteration.  This is often  not the case in practice since data samples of Q-learning are drawn according to the underlying Markov decision process.  Till now,  the  finite-time convergence rate of \qwhittle under the more challenging Markovian setting remains to be an open problem. Though \cite{pagare2023full} proposed a novel DQN method and applied it to Whittle index learning, it lacks of theoretical convergence analysis. To our best knowledge, our work is the first to study low-complexity model-free Q-learning for RMAB with neural network function approximation and provide a finite-time performance guarantee.


\textbf{Two-Timescale Stochastic Approximation.}
The theoretical understanding of average-reward reinforcement learning (RL) methods is limited.  Most existing results focus on asymptotic convergence \cite{tsitsiklis1999average,abounadi2001learning,wan2021learning,zhang2021average}, or finite-time performance guarantee for discounted Q-learning \cite{chen2019performance,qu2020finite,chen2020finite}. However, the analysis of average-reward RL algorithms is known to be more challenging than their discounted-reward counterparts \cite{zhang2021finite,wei2020model}.  In particular, our \qwhittleNFA  follows the 2TSA scheme \cite{borkar1997actor,konda2000actor,bhatnagar2009natural}. The standard technique for analyzing 2TSA is via the ODE method to prove asymptotic convergence  \cite{borkar2009stochastic}. Building off the importance of asymptotic results, recent years have witnessed a focus shifted to non-asymptotic, finite-time analysis of 2TSA \cite{gupta2019finite,doan2019linear,doan2020nonlinear,yang2019provably}. The closest work is \cite{doan2020nonlinear}, which characterized the convergence rate for a general non-linear 2TSA with i.i.d. noise.  We generalize this result to provide a finite-time analysis of our \qwhittleNFA with Markovian noise.
In addition, existing finite-time analysis, e.g., sample complexity \cite{zhang2021finite} and regret \cite{wei2020model} of Q-learning with average reward focus on a single-timescale SA, and hence cannot be directly applied to our \qwhittleNFA. Finally, existing Q-learning with linear function approximation \cite{melo2008analysis,bhandari2018finite,zou2019finite} and neural network function approximation \cite{cai2023neural,xu2020finite} requires an additional projection step onto a bounded set related to the unknown stationary distribution of the underlying MDPs, or focuses on a single-timescale SA \cite{chen2019performance}.

\section{Review on Whittle Index Policy}\label{sec:whittle-app}
Whittle index policy addresses the intractable issue of RMAB through decomposition. In each round $t$,
it first calculates the Whittle index for each arm $n$ independently only based on its current state $s_n(t)$, and then the Whittle index policy simply selects the $K$ arms with the highest indices to activate.
Following Whittle's approach\cite{whittle1988restless}, we can consider a system with only one arm due to the decomposition, and the Lagrangian is expressed as
\begin{align}\label{eq:lagrangian}
 L(\pi,\lambda)&=\liminf_{T\rightarrow\infty}\frac{1}{T}\mathbb{E}_\pi\sum_{t=1}^T\Big\{r(t)+\lambda\Big(1- a(t)\Big)\Big\},
\end{align}
where $\lambda$ is the Lagrangian multiplier (or the subsidy for selecting passive action). For a particular $\lambda$, the optimal activation policy can be expressed by a set of states in which it would activate this arm, which is denoted $D(\lambda)$.

\begin{defn}[Indexiability]
We denote $D(\lambda)$ as the set of states $S$ for which the optimal action for the  arm is to choose a passive action, i.e., $A=0$.  Then the arm is said to be indexable if $D(\lambda)$ increases with $\lambda$, i.e., if $\lambda>\lambda^\prime$, then $D(\lambda)\supseteq D(\lambda^{\prime})$.
\end{defn}

Following the indexability property, the Whittle index in a particular state $S$ is defined as follows.

\begin{defn} [Whittle Index]
The Whittle index in state $S$ for the indexable arm is the smallest value of the Lagrangian multiplier $\lambda$  such that the optimal policy at state $S$ is indifferent towards actions $A=0$ and $A=1$.  We denote such a Whittle index as $\lambda(S)$ satisfying $\lambda(S):=\inf_{\lambda\geq 0}\{S\in D(\lambda)\}$.
\end{defn}

\begin{defn}[Whittle index policy]
Whittle index policy is a controlled policy which activates the $K$ arms with the highest whittle index $\lambda_i(S_i(t))$ at each time slot $t$.

\end{defn}




\section{Proof of Lemmas for ``Step 2: Convergence Rate of $\hat{M}(\pmb{\theta}_k, \lambda_k)$ in (\ref{eq:lyapunov-function2})''}

\subsection{Proof of Lemma \ref{lem:h_lipschitz}}
\begin{proof}

Recall that
$$
f_0(\pmb{\theta}; \pmb{\phi}(s,a))=\frac{1}{\sqrt{m}}\sum_{r=1}^m b_r\mathds{1}\{\mathbf{w}_{r,0}^\intercal\pmb{\phi}(s,a) >0\}\mathbf{w}_r^\intercal\pmb{\phi}(s,a).
$$
Thus we denote $\nabla_{\pmb{\theta}}f_0(\pmb{\theta};\pmb{\phi}(s, a))$ as
\begin{align}
    \label{eq:gradient}
    \nabla_{\pmb{\theta}}f_0(\pmb{\theta};\pmb{\phi}(s, a)):=\Big[&\frac{1}{\sqrt{m}} b_1\mathds{1}\{\mathbf{w}_{1,0}^\intercal\pmb{\phi}(s,a) >0\}\pmb{\phi}(s,a)^\intercal,\ldots,\nonumber\allowdisplaybreaks\\
    &\qquad\qquad\frac{1}{\sqrt{m}} b_m\mathds{1}\{\mathbf{w}_{m,0}^\intercal\pmb{\phi}(s,a) >0\}\pmb{\phi}(s,a)^\intercal\Big]^\intercal.
\end{align}
Since $\|\pmb{\phi}(s,a)\|\leq 1, \forall s\in\gS, a\in\gA $ and the fact that $b_r, \forall r\in[m]$ is uniformly initialized as $1$ and $-1$, we have $\|\nabla_{\pmb{\theta}}f_0(\pmb{\theta};\pmb{\phi}(s, a))\|\leq 1$.

Therefore,  we have the following inequality for any parameter pairs $(\pmb{\theta}_1, \lambda_1)$ and $(\pmb{\theta}_2, \lambda_2)$ with $X=(s,a,s^\prime)\in\gX$,
\begin{align*}
   &\| h_0(X,\pmb{\theta}_1, \lambda_1)- h_0(X,\pmb{\theta}_2, \lambda_2)\|\allowdisplaybreaks\\
   &=\Big\|\nabla_{\pmb{\theta}}f_0(\pmb{\theta}_1;\pmb{\phi}(s, a))\Big[r(s,a)+(1-a)\lambda_1-I_0( \pmb{\theta}_1)+\max_{a_1}f_0(\pmb{\theta}_1; \pmb{\phi}(s^\prime,a_1))-f_0(\pmb{\theta}_1; \pmb{\phi}(s,a))\Big]\allowdisplaybreaks\\
   &-\nabla_{\pmb{\theta}}f_0(\pmb{\theta}_2;\pmb{\phi}(s, a))\Big[r(s,a)+(1-a)\lambda_2-I_0( \pmb{\theta}_2)+\max_{a_2}f_0(\pmb{\theta}_2; \pmb{\phi}(s^\prime,a_2))-f_0(\pmb{\theta}_2; \pmb{\phi}(s,a))\Big]\Big\|\allowdisplaybreaks\\
    &\overset{(a_1)}{=} \Big\|\nabla_{\pmb{\theta}}f_0(\pmb{\theta}_1;\pmb{\phi}(s, a))\Big[(1-a)(\lambda_1-\lambda_2)+I_0( \pmb{\theta}_2)-I_0( \pmb{\theta}_1)+f_0(\pmb{\theta}_2; \pmb{\phi}(s,a))-f_0(\pmb{\theta}_1; \pmb{\phi}(s,a))\nonumber\allowdisplaybreaks\\
    &\qquad\qquad\qquad\qquad+\max_{a_1}\Big(f_0(\pmb{\theta}_1; \pmb{\phi}(s^\prime,a_1))-\max_{a_2}f_0(\pmb{\theta}_2; \pmb{\phi}(s^\prime,a_2))\Big]\Big\|\allowdisplaybreaks\\
   &\overset{(a_2)}{\leq} \|(1-a)(\lambda_1-\lambda_2)\|+\|f_0(\pmb{\theta}_2;\pmb{\phi}(s, a))-f_0(\pmb{\theta}_1;\pmb{\phi}(s, a))\|\allowdisplaybreaks\\
   &\qquad\qquad+\Bigg\|\frac{1}{2S}\sum_{\tilde{s}\in\gS}f_0(\pmb{\theta}_2;\pmb{\phi}(\tilde{s}, 0))-f_0(\pmb{\theta}_1;\pmb{\phi}(\tilde{s}, 0))+f_0(\pmb{\theta}_2;\pmb{\phi}(\tilde{s}, 1))-f_0(\pmb{\theta}_1;\pmb{\phi}(\tilde{s}, 1))\Bigg\|\allowdisplaybreaks\\
   &\qquad\qquad+\Big\|\max_{a_1}\Big(f_0(\pmb{\theta}_1; \pmb{\phi}(s^\prime,a_1))-\max_{a_2}f_0(\pmb{\theta}_2; \pmb{\phi}(s^\prime,a_2))\Big\|\allowdisplaybreaks\\
   &\overset{(a_3)}{\leq} \|(1-a)(\lambda_1-\lambda_2)\|+\|\nabla_{\pmb{\theta}}f_0(\pmb{\theta}_1;\pmb{\phi}(s, a))(\pmb{\theta}_2-\pmb{\theta}_1)\|\allowdisplaybreaks\\
   &\qquad\qquad\qquad\qquad+\Bigg\|\frac{1}{2S}\sum_{\tilde{s}\in\gS}\nabla_{\pmb{\theta}}f_0(\pmb{\theta}_1;\pmb{\phi}(\tilde{s}, 0))(\pmb{\theta}_2-\pmb{\theta}_1)+\nabla_{\pmb{\theta}}f_0(\pmb{\theta}_1;\pmb{\phi}(\tilde{s}, 1))(\pmb{\theta}_2-\pmb{\theta}_1)\Bigg\|\allowdisplaybreaks\\
   &\qquad\qquad\qquad\qquad+\Big\|\max_{a_1}\Big(f_0(\pmb{\theta}_1; \pmb{\phi}(s^\prime,a_1))-\max_{a_2}f_0(\pmb{\theta}_2; \pmb{\phi}(s^\prime,a_2))\Big\|\allowdisplaybreaks\\
   &\overset{(a_4)}{\leq} \|(\lambda_1-\lambda_2)\|+2\| \pmb{\theta}_1-\pmb{\theta}_2\|+\Big\|\max_{a_1}\Big(f_0(\pmb{\theta}_1; \pmb{\phi}(s^\prime,a_1))-\max_{a_2}f_0(\pmb{\theta}_2; \pmb{\phi}(s^\prime,a_2))\Big\|\allowdisplaybreaks\\
   &\overset{(a_5)}{\leq} \|(\lambda_1-\lambda_2)\|+2\| \pmb{\theta}_1-\pmb{\theta}_2\|+\Big\|\max_{a^\prime}f_0(\pmb{\theta}_1; \pmb{\phi}(s^\prime,a^\prime))-f_0(\pmb{\theta}_2; \pmb{\phi}(s^\prime,a^\prime))\Big\|\allowdisplaybreaks\\
   &\overset{(a_6)}{\leq} \|(\lambda_1-\lambda_2)\|+3\| \pmb{\theta}_1-\pmb{\theta}_2\|.
\end{align*}
Specifically, $(a_1)$ holds due to the fact that $\nabla_{\pmb{\theta}}f_0(\pmb{\theta}_1;\pmb{\phi}(s, a))=\nabla_{\pmb{\theta}}f_0(\pmb{\theta}_2;\pmb{\phi}(s, a))$ as in (\ref{eq:gradient}). Since $$I_0(\pmb{\theta}_k)=\frac{1}{2S}\sum_{\tilde{s}\in\gS}\Big[f_0(\pmb{\theta}_k; \pmb{\phi}(\tilde{s},0))+f_0(\pmb{\theta}_k; \pmb{\phi}(\tilde{s},1))\Big],$$
$(a_2)$ is due to the fact that $\|\vx+\vy\|\leq \|\vx\|+\|\vy\|, \forall \vx,\vy\in\mathbb{R}^{md}$ and $\|\vx\cdot\vy\|\leq \|\vx\|\cdot\|\vy\|, \forall \vx,\vy\in\mathbb{R}^{md}$ and $\|\pmb{\phi}(s,a)\|\leq 1, \forall s,a.$ $(a_3)$ holds since
\begin{align}\label{eq:f_difference}
  f_0(\pmb{\theta}_2;\pmb{\phi}(s, a))-f_0(\pmb{\theta}_1;\pmb{\phi}(s, a))=\nabla_{\pmb{\theta}}f_0(\pmb{\theta}_1;\pmb{\phi}(s, a))(\pmb{\theta}_2-\pmb{\theta}_1), \forall s\in\gS, a\in\gA.
\end{align}
$(a_4)$ holds for the same reason as $(a_2)$. $(a_5)$ is due to the fact that
\begin{align}
    \label{eq:max_f_difference}\nonumber
    \|\max_{a^\prime}f_0(\pmb{\theta}_1; \pmb{\phi}(s^\prime,a^\prime))-f_0(\pmb{\theta}_2; \pmb{\phi}(s^\prime,a^\prime))\|&\leq \max\Bigg(\Big\|\max_{a^\prime}f_0(\pmb{\theta}_1; \pmb{\phi}(s^\prime,a^\prime))-f_0(\pmb{\theta}_2; \pmb{\phi}(s^\prime,a^\prime))\Big\|,\\ &\Big\|\min_{a^\prime}f_0(\pmb{\theta}_1; \pmb{\phi}(s^\prime,a^\prime))-f_0(\pmb{\theta}_2; \pmb{\phi}(s^\prime,a^\prime))\Big\|\Bigg).
\end{align}
$(a_6)$ holds for the same reason as $(a_3)$ and $(a_4)$.
\end{proof}

\subsection{Proof of Lemma \ref{lem:g_lipschitz}}
\begin{proof}
Since $g_0(\cdot)$ is irrelevant with $X$ and $\lambda$, in the following, we write $g_0(X,\pmb{\theta},\lambda)$ with $g_0(\lambda)$ interchangeably.
For any $\pmb{\theta}_1\in\mathbb{R}^{md}$ and $\pmb{\theta}_2\in\mathbb{R}^{md}$, we have
\begin{align*}
   \|&g_0 (\pmb{\theta}_1)-g_0(\pmb{\theta}_2)\|\\
    &=\|f_0(\pmb{\theta}_1;\pmb{\phi}(s,1))-f_0(\pmb{\theta}_1;\pmb{\phi}(s,0))-f_0(\pmb{\theta}_2;\pmb{\phi}(s,1))+f_0(\pmb{\theta}_2;\pmb{\phi}(s,0))\|\\
    &\leq \|f_0(\pmb{\theta}_1;\pmb{\phi}(s,1))-f_0(\pmb{\theta}_2;\pmb{\phi}(s,1))\|+\|f_0(\pmb{\theta}_1;\pmb{\phi}(s,0))-f_0(\pmb{\theta}_2;\pmb{\phi}(s,0))\|\\
    &= \|\nabla_{\pmb{\theta}}f_0(\pmb{\theta}_1;\pmb{\phi}(s, 1))(\pmb{\theta}_2-\pmb{\theta}_1)\|+\|\nabla_{\pmb{\theta}}f_0(\pmb{\theta}_1;\pmb{\phi}(s, 0))(\pmb{\theta}_2-\pmb{\theta}_1)\|\\
    &\leq \|\nabla_{\pmb{\theta}}f_0(\pmb{\theta}_1;\pmb{\phi}(s, 1))\|\cdot\| \pmb{\theta}_1-\pmb{\theta}_2\|+\|\nabla_{\pmb{\theta}}f_0(\pmb{\theta}_1;\pmb{\phi}(s, 0))\|\cdot\| \pmb{\theta}_1-\pmb{\theta}_2\|\\
    &\leq 2\|\pmb{\theta}_1-\pmb{\theta}_2\|,
\end{align*}
where the first inequality is due to the fact that
$\|\vx+\vy\|\leq\|\vx\|+\|\vy\|$, $\forall \vx, \vy\in\mathbb{R}^{md}$, the second inequality holds due to
$\|\vx\cdot\vy\|\leq\|\vx\|\cdot\|\vy\|,$ $\forall \vx, \vy\in\mathbb{R}^{md}$, and the last inequality holds since $\|\nabla_{\pmb{\theta}}f_0(\pmb{\theta}_1;\pmb{\phi}(s, a))\|\leq 1, \forall s\in\gS, a\in\gA.$
\end{proof}

\subsection{Proof of Lemma \ref{lem:y_lipschitz}}

\begin{proof}
For any $\pmb{\theta}_1\in\mathbb{R}^{md}$ and $\pmb{\theta}_2\in\mathbb{R}^{md}$, we have
\begin{align*}
&\|y_0(\pmb{\theta}_1)-y_0(\pmb{\theta}_2)\|\\
&=\Big\|r(s,1)-r(s,0)+\sum_{s^\prime}P(s^\prime|s,1)\max_a f_0(\pmb{\theta}_1;\pmb{\phi}(s^\prime,a))-\sum_{s^\prime}P(s^\prime|s,0)\max_a f_0(\pmb{\theta}_1;\pmb{\phi}(s^\prime,a))\\
&-r(s,1)-r(s,0)+\sum_{s^\prime}P(s^\prime|s,1)\max_a f_0(\pmb{\theta}_2;\pmb{\phi}(s^\prime,a))-\sum_{s^\prime}P(s^\prime|s,0)\max_a f_0(\pmb{\theta}_2;\pmb{\phi}(s^\prime,a))\Big\|\\
&=\Big\|\sum_{s^\prime}P(s^\prime|s,1)\max_a f_0(\pmb{\theta}_1;\pmb{\phi}(s^\prime,a))-\sum_{s^\prime}P(s^\prime|s,1)\max_a f_0(\pmb{\theta}_2;\pmb{\phi}(s^\prime,a))\\
&\qquad\qquad-\sum_{s^\prime}P(s^\prime|s,0)\max_a f_0(\pmb{\theta}_1;\pmb{\phi}(s^\prime,a))+\sum_{s^\prime}P(s^\prime|s,0)\max_a f_0(\pmb{\theta}_2;\pmb{\phi}(s^\prime,a))\Big\|\\
&\leq \Big\|\sum_{s^\prime}P(s^\prime|s,1)\max_a f_0(\pmb{\theta}_1;\pmb{\phi}(s^\prime,a))-\sum_{s^\prime}P(s^\prime|s,1)\max_a f_0(\pmb{\theta}_2;\pmb{\phi}(s^\prime,a))\Big\|\\
&\qquad\qquad+\Big\|\sum_{s^\prime}P(s^\prime|s,0)\max_a f_0(\pmb{\theta}_1;\pmb{\phi}(s^\prime,a))+\sum_{s^\prime}P(s^\prime|s,0)\max_a f_0(\pmb{\theta}_2;\pmb{\phi}(s^\prime,a))\Big\|\\
&\leq 2\|\pmb{\theta}_1-\pmb{\theta}_2\|,
\end{align*}
with the last inequality holds due to (\ref{eq:f_difference}) and (\ref{eq:max_f_difference}).
\end{proof}

\subsection{Proof of Lemma \ref{lem:gh_monotone}}
\begin{proof}
$1)$ We first show that there exists a constant $\mu_1>0$ such that $\mathbb{E}[\hat{\pmb{\theta}}^\intercal h_0(X,\pmb{\theta}, {\lambda})]\leq -\mu_1\|\hat{\pmb{\theta}}\|^2$. According to the definition of $\pmb{\theta}_0^*$ given in Definition \ref{def:stationary_point}, $\mathbb{E}[h_0(X,\pmb{\theta}_{0}^*, y_0(\pmb{\theta}_0^*))]=0$. Hence, we have
\begin{align*}
&\mathbb{E}\left[\hat{\pmb{\theta}}^\intercal (h_0(X,\pmb{\theta}, {\lambda})-h_0(X,\pmb{\theta}_0^*, y_0(\pmb{\theta}_0^*))\right]\allowdisplaybreaks\\
   &=\hat{\pmb{\theta}}^\intercal\mathbb{E}[h_0(X,\pmb{\theta}, \lambda)- h_0(X,\pmb{\theta}_0^*, y_0(\pmb{\theta}_0^*))]\allowdisplaybreaks\\
   &=\hat{\pmb{\theta}}^\intercal\mathbb{E}\Big[\nabla_{\pmb{\theta}}f_0(\pmb{\theta};\pmb{\phi}(s, a))\Big[r(s,a)+(1-a)\lambda-I_0( \pmb{\theta})+\max_{a_1}f_0(\pmb{\theta}; \pmb{\phi}(s^\prime,a_1))-f_0(\pmb{\theta}; \pmb{\phi}(s,a))\Big]\allowdisplaybreaks\\
   &\quad-\nabla_{\pmb{\theta}}f_0(\pmb{\theta}_0^*;\pmb{\phi}(s, a))\Big[r(s,a)+(1-a)y_0(\pmb{\theta}_0^*)-I_0( \pmb{\theta}_0^*)+\max_{a_2}f_0(\pmb{\theta}_0^*; \pmb{\phi}(s^\prime,a_2))-f_0(\pmb{\theta}_0^*; \pmb{\phi}(s,a))\Big]\Big]\allowdisplaybreaks\\
    &\overset{(b_1)}{=} \hat{\pmb{\theta}}^\intercal\mathbb{E}\Big[\nabla_{\pmb{\theta}}f_0(\pmb{\theta};\pmb{\phi}(s, a))\Big[(1-a)(\lambda-y_0(\pmb{\theta}_0^*))+I( \pmb{\theta}_0^*)-I_0( \pmb{\theta})+f_0(\pmb{\theta}_0^*; \pmb{\phi}(s,a))-f_0(\pmb{\theta}; \pmb{\phi}(s,a))\nonumber\allowdisplaybreaks\\
    &\qquad\qquad\qquad\qquad+\max_{a_1}f_0(\pmb{\theta}; \pmb{\phi}(s^\prime,a_1))-\max_{a_2}f_0(\pmb{\theta}_0^*; \pmb{\phi}(s^\prime,a_2))\Big]\Big]\allowdisplaybreaks\\
     &\quad{=} ~\hat{\pmb{\theta}}^\intercal\mathbb{E}\Big[\nabla_{\pmb{\theta}}f_0(\pmb{\theta};\pmb{\phi}(s, a))\Big[\max_{a_1}f_0(\pmb{\theta}; \pmb{\phi}(s^\prime,a_1))-\max_{a_2}f_0(\pmb{\theta}_0^*; \pmb{\phi}(s^\prime,a_2))\Big]\Big]\nonumber\allowdisplaybreaks\\
    &\quad
    -\hat{\pmb{\theta}}^\intercal\mathbb{E}\Big[\nabla_{\pmb{\theta}}f_0(\pmb{\theta};\pmb{\phi}(s, a))[I_0( \pmb{\theta})-I_0( \pmb{\theta}_0^*)]\Big]-\hat{\pmb{\theta}}^\intercal\mathbb{E}\Big[\nabla_{\pmb{\theta}}f_0(\pmb{\theta};\pmb{\phi}(s, a))[f_0(\pmb{\theta}; \pmb{\phi}(s,a))-f_0(\pmb{\theta}_0^*; \pmb{\phi}(s,a))]\Big]\nonumber\allowdisplaybreaks\\
    &\qquad\qquad\qquad\qquad+\hat{\pmb{\theta}}^\intercal\mathbb{E}\Big[\nabla_{\pmb{\theta}}f_0(\pmb{\theta};\pmb{\phi}(s, a))[(1-a)(\lambda-y_0(\pmb{\theta}_0^*))]\Big]\allowdisplaybreaks\\
    &\overset{(b_2)}{\leq} \hat{\pmb{\theta}}^\intercal\mathbb{E}\Big[\nabla_{\pmb{\theta}}f_0(\pmb{\theta};\pmb{\phi}(s, a))\max_{a^\prime}\Big[f_0(\pmb{\theta}; \pmb{\phi}(s^\prime,a^\prime))-f_0(\pmb{\theta}_0^*; \pmb{\phi}(s^\prime,a^\prime))\Big]\Big]\nonumber\allowdisplaybreaks\\
    &\quad
    -\hat{\pmb{\theta}}^\intercal\mathbb{E}\Big[\nabla_{\pmb{\theta}}f_0(\pmb{\theta};\pmb{\phi}(s, a))[I_0( \pmb{\theta})-I_0( \pmb{\theta}_0^*)]\Big]-\hat{\pmb{\theta}}^\intercal\mathbb{E}\Big[\nabla_{\pmb{\theta}}f_0(\pmb{\theta};\pmb{\phi}(s, a))[f_0(\pmb{\theta}; \pmb{\phi}(s,a))-f_0(\pmb{\theta}_0^*; \pmb{\phi}(s,a))]\Big]\nonumber\allowdisplaybreaks\\
    &\qquad\qquad\qquad\qquad+\hat{\pmb{\theta}}^\intercal\mathbb{E}\Big[\nabla_{\pmb{\theta}}f_0(\pmb{\theta};\pmb{\phi}(s, a))[(1-a)(\lambda-y_0(\pmb{\theta}_0^*))]\Big]\allowdisplaybreaks\\
     &\overset{(b_3)}{\leq} \hat{\pmb{\theta}}^\intercal\mathbb{E}\Big[\nabla_{\pmb{\theta}}f_0(\pmb{\theta};\pmb{\phi}(s, a))\max_{a^\prime}\Big[f_0(\pmb{\theta}; \pmb{\phi}(s^\prime,a^\prime))-f_0(\pmb{\theta}_0^*; \pmb{\phi}(s^\prime,a^\prime))\Big]\Big]\nonumber\allowdisplaybreaks\\
    &\quad
    -\hat{\pmb{\theta}}^\intercal\mathbb{E}\Big[\nabla_{\pmb{\theta}}f_0(\pmb{\theta};\pmb{\phi}(s, a))[I_0( \pmb{\theta})-I_0( \pmb{\theta}_0^*)]\Big]-\hat{\pmb{\theta}}^\intercal\mathbb{E}\Big[\nabla_{\pmb{\theta}}f_0(\pmb{\theta};\pmb{\phi}(s, a))[f_0(\pmb{\theta}; \pmb{\phi}(s,a))-f_0(\pmb{\theta}_0^*; \pmb{\phi}(s,a))]\Big]\nonumber\allowdisplaybreaks\\
   &\overset{(b_4)}{=} \|\hat{\pmb{\theta}}\|^2\mathbb{E}\Big[\nabla_{\pmb{\theta}}f_0(\pmb{\theta};\pmb{\phi}(s, a))^\intercal \nabla_{\pmb{\theta}}f_0(\pmb{\theta}; \pmb{\phi}(s^\prime,\tilde{a}))\Big]\nonumber\allowdisplaybreaks\\
    &\qquad\qquad\qquad
    -\|\hat{\pmb{\theta}}\|^2\mathbb{E}\Big[\nabla_{\pmb{\theta}}f_0(\pmb{\theta};\pmb{\phi}(s, a))^\intercal\Big[\frac{1}{2S}\sum_{\tilde{s}\in\gS}\nabla_{\pmb{\theta}}f_0(\pmb{\theta};\pmb{\phi}(\tilde{s},0))+f_0(\pmb{\theta};\pmb{\phi}(\tilde{s},1))\Big]\Big]\nonumber\allowdisplaybreaks\\
    &\qquad\qquad\qquad-\|\hat{\pmb{\theta}}\|^2\mathbb{E}\Big[\nabla_{\pmb{\theta}}f_0(\pmb{\theta};\pmb{\phi}(s, a))^\intercal \nabla_{\pmb{\theta}}f_0(\pmb{\theta};\pmb{\phi}(s, a))\Big]\nonumber\allowdisplaybreaks\\
\end{align*}
where $(b_1)$ holds since $\nabla_{\pmb{\theta}}f_0(\pmb{\theta}; \pmb{\phi}(s,a))=\nabla_{\pmb{\theta}}f_0(\pmb{\theta}_0^*; \pmb{\phi}(s,a))$ as in (\ref{eq:gradient}), $(b_2)$ is due to the fact that  $\max_{a_1}f_0(\pmb{\theta}; \pmb{\phi}(s^\prime,a_1))-\max_{a_1}f_0(\pmb{\theta}_0^*; \pmb{\phi}(s^\prime,a_2))\leq \max_{a^\prime}\Big[f_0(\pmb{\theta}; \pmb{\phi}(s^\prime,a^\prime))-f_0(\pmb{\theta}_0^*; \pmb{\phi}(s^\prime,a^\prime))\Big]$, and $(b_3)$ holds due to the fact that  $\hat{\pmb{\theta}}^\intercal\mathbb{E}\Big[\nabla_{\pmb{\theta}}f_0(\pmb{\theta};\pmb{\phi}(s, a))[(1-a)(\lambda-y_0(\pmb{\theta}_0^*))]\Big]\leq 0$ since a larger Whittle index $\lambda$ will choose the action $a=1$.  Notice that the $\tilde{a}$ in $(b_4)$ represents the action $a^\prime$ which maximizes  $f_0(\pmb{\theta}; \pmb{\phi}(s^\prime,a^\prime))-f_0(\pmb{\theta}_0^*; \pmb{\phi}(s^\prime,a^\prime))$. Due to the definition of $\nabla_{\pmb{\theta}}f_0(\pmb{\theta}; \pmb{\phi}(s,a))$ in (\ref{eq:gradient}), we show that $\mathbb{E}[\hat{\pmb{\theta}}^\intercal h_0(X,\pmb{\theta}, {\lambda})]\leq 0$.

$2)$ Next, we show that there exists a constant $\mu_2>0$ such that $\mathbb{E}[ \hat{{\lambda}} g_0(X,\pmb{\theta}, {\lambda})]\leq -\mu_2\|\hat{{\lambda}}\|^2.$ According to the definition of $g_0(\pmb{\theta}),$ i.e., $g_0(\pmb{\theta}):=f_0(\pmb{\theta}; \pmb{\phi}(s,1))-f_0(\pmb{\theta}; \pmb{\phi}(s,0))$. Since $y_0(\pmb{\theta})$ is the solution of $\lambda$ such that $f_0(\pmb{\theta}; \pmb{\phi}(s,1))=f_0(\pmb{\theta}; \pmb{\phi}(s,0))$, the signs of $\hat{\lambda}:=\lambda-y_0(\theta)$ and $f_0(\pmb{\theta}; \pmb{\phi}(s,1))-f_0(\pmb{\theta}; \pmb{\phi}(s,0))$ are always opposite. Hence, we have $\mathbb{E}[ \hat{{\lambda}} g_0(X,\pmb{\theta}, {\lambda})]\leq 0,$ which completes the proof.

\end{proof}

\subsection{Proof of Lemma \ref{lemma:mixing_time}}
 \begin{proof}

Under Lemma \ref{lem:h_lipschitz}, we have
 \begin{align}\label{eq:lemma5-1}
     \|h_0(X, \pmb{\theta}, {\lambda})-h_0(X, \pmb{\theta}^*, {\lambda}^*)\|\leq 3\|\pmb{\theta}-\pmb{\theta}^*\|+\|\pmb{\lambda}-\pmb{\lambda}^*\|.
 \end{align}
 Let $L=\max(3, \max_X h_0(X, \pmb{\theta}^*, {\lambda}^*))$, then according to (\ref{eq:lemma5-1}), we have
\begin{align*}
     \|h_0(X, \pmb{\theta}, {\lambda})\|\leq L(\|\pmb{\theta}-\pmb{\theta}^*\|+\|\pmb{\lambda}-\pmb{\lambda}^*\|+1).
 \end{align*}

 Denote $h_0^i(X,\pmb{\theta}, {\lambda})$ as the $i$-th element of $h_0(X,\pmb{\theta}, {\lambda})$. Following \cite{chen2019performance}, we can show that $\pmb{\theta}\in\mathbb{R}^{md}$, $
{\lambda}\in\mathbb{R}^{1}$, and $x\in\gX$,
 \begin{align*}
     &\|\mathbb{E}[h_0(X_k, \pmb{\theta}, \lambda)|X_0=x]-\mathbb{E}_\mu[h_0(X, \pmb{\theta}, \lambda)]\|\\
    & \leq \sum_{i=1}^{md}|\mathbb{E}[h_i(X_k, \pmb{\theta}, \lambda)|X_0=x]-\mathbb{E}_\mu[h_0^i(X, \pmb{\theta}, \lambda)]|\\
    &\leq 2L(\|\pmb{\theta}-\pmb{\theta}^*\|+\|{\lambda}-{\lambda}^*\|+1)\sum_{i=1}^{md}\Bigg|\mathbb{E}\left[\frac{h_0^i(X_k, \pmb{\theta}, \lambda)}{2L(\|\pmb{\theta}-\pmb{\theta}^*\|+\|{\lambda}-{\lambda}^*\|+1)}\Big|X_0=x\right]\\
    &\qquad\qquad\qquad\qquad-\mathbb{E}_\mu\left[\frac{h_0^i(X, \pmb{\theta}, \lambda)}{2L(\|\pmb{\theta}-\pmb{\theta}^*\|+\|{\lambda}-{\lambda}^*\|+1)}\right]\Bigg|\\
    &\leq 2L(\|\pmb{\theta}-\pmb{\theta}^*\|+\|{\lambda}-{\lambda}^*\|+1)mdC\rho^k,
\end{align*}
where the last inequality holds due to Assumption \ref{assumption:markovian}.
To guarantee $2L(\|\pmb{\theta}-\pmb{\theta}^*\|+\|{\lambda}-{\lambda}^*\|+1)mdC\rho^k\leq \delta(\|\pmb{\theta}-\pmb{\theta}^*\|+\|{\lambda}-{\lambda}^*\|+1)$, we have
\begin{align*}
   \tau_\delta\leq \frac{\log(1/\delta)+\log(2LCmd)}{\log(1/\rho)},
\end{align*}
which completes the proof.

\end{proof}

\subsection{Proof of Lemma \ref{lemma:M-and-M_hat}}

\begin{proof}
Based on the definition of ${M}(\pmb{\theta}_k, {\lambda}_k)$ in (\ref{eq:lyapunov-function}), we have
    \begin{align}\label{eq:lemma6}
   {M}(\pmb{\theta}_k, {\lambda}_k)&:=\frac{\eta_k}{\alpha_k}\|\pmb{\theta}_k-\pmb{\theta}^*\|^2+\|{\lambda}_k-y(\pmb{\theta}_k)\|^2\nonumber\allowdisplaybreaks\\
   &=\frac{\eta_k}{\alpha_k}\|\pmb{\theta}_k-\pmb{\theta}_0^*+\pmb{\theta}_0^*-\pmb{\theta}^*\|^2+\|{\lambda}_k-{y}_0(\pmb{\theta}_k)+{y}_0(\pmb{\theta}_k)-y(\pmb{\theta}_k)\|^2\nonumber\allowdisplaybreaks\\
   &\leq \frac{2\eta_k}{\alpha_k}(\|\pmb{\theta}_k-{\pmb{\theta}}_0^*\|^2+\|{\pmb{\theta}}_0^*-\pmb{\theta}^*\|^2)+2(\|{\lambda}_k-y_0(\pmb{\theta}_k)\|^2+\|y_0(\pmb{\theta}_k)-y(\pmb{\theta}_k)\|^2)\nonumber\allowdisplaybreaks\\
   &=2\hat{M}(\pmb{\theta}_k, {\lambda}_k)+\frac{2\eta_k}{\alpha_k}\|{\pmb{\theta}}_0^*-\pmb{\theta}^*\|^2+2\|{y}_0(\pmb{\theta}_k)-y(\pmb{\theta}_k)\|^2\nonumber\allowdisplaybreaks\\
   &\leq 2\hat{M}(\pmb{\theta}_k, {\lambda}_k)+\frac{2\eta_kc_0^2}{\alpha_k}\|span({f}_0({\pmb{\theta}}_0^*)-{f}({\pmb{\theta}^*}))\|^2+2\|{y}_0(\pmb{\theta}_k)-y(\pmb{\theta}_k)\|^2,
\end{align}
where the first inequality holds based on $\|\vx+\vy\|^2\leq 2\|\vx\|^2+2\|\vy\|^2$, and the second inequality holds based on Assumption \ref{asump:span}.
Next, we bound $\|span({f}_0({\pmb{\theta}}_0^*)-{f}({\pmb{\theta}^*}))\|$ as follows
\begin{align}
  \label{eq:lemma6-1}\|span(f_0({\pmb{\theta}}_0^*)-{f}({\pmb{\theta}^*}))\|&=   \|span(f_0({\pmb{\theta}}_0^*)-\Pi_\mathcal{F}{f}({\pmb{\theta}^*})+\Pi_\mathcal{F}{f}({\pmb{\theta}^*})-{f}({\pmb{\theta}^*}))\|\nonumber\allowdisplaybreaks\\
  &\leq\|span(f_0({\pmb{\theta}}_0^*)-\Pi_\mathcal{F}{f}({\pmb{\theta}^*}))\|+\|span(\Pi_\mathcal{F}{f}({\pmb{\theta}^*})-{f}({\pmb{\theta}^*}))\|\nonumber\allowdisplaybreaks\\
  &=\|span(\Pi_{\mathcal{F}}\mathcal{T}f_0({\pmb{\theta}}_0^*)-\Pi_\mathcal{F}\mathcal{T}{f}({\pmb{\theta}^*}))\|+\|span(\Pi_\mathcal{F}{f}({\pmb{\theta}^*})-{f}({\pmb{\theta}^*}))\|\nonumber\allowdisplaybreaks\\
  &\leq \kappa \|span(f_0({\pmb{\theta}}_0^*)-{f}({\pmb{\theta}^*}))\|+\|span(\Pi_\mathcal{F}{f}({\pmb{\theta}^*})-{f}({\pmb{\theta}^*}))\|,
\end{align}
where the last inequality follows (\ref{eq:span_contraction}).
This indicates that
\begin{align}
   \label{eq:lemma6-2}\|span(f_0({\pmb{\theta}}_0^*)-{f}({\pmb{\theta}^*}))\|^2\leq \frac{1}{(1-\kappa)^2} \|span(\Pi_\mathcal{F}{f}({\pmb{\theta}^*})-{f}({\pmb{\theta}^*}))\|^2.
\end{align}

We further bound $\|{y}_0(\pmb{\theta}_k)-y(\pmb{\theta}_k)\|^2$ as follows
\begin{align}
\label{eq:lemma6-3}
   \|y_0(\pmb{\theta}_k)-y(\pmb{\theta}_k)\|^2&=\Big\|\sum_{s^\prime}p(s^\prime |s,\!1)\max_{a_1} f_0(\pmb{\theta}_k;\pmb{\phi}(s^\prime,a_1)) -\sum_{s^\prime}p(s^\prime |s,\!0)\max_{a_2} f_0(\pmb{\theta}_k;\pmb{\phi}(s^\prime,a_2))\nonumber\allowdisplaybreaks\\
   &-\sum_{s^\prime}p(s^\prime |s,\!1)\max_{a_3} {f}(\pmb{\theta}_k;\pmb{\phi}(s^\prime,a_3)) +\sum_{s^\prime}p(s^\prime |s,\!0)\max_{a_4} {f}(\pmb{\theta}_k;\pmb{\phi}(s^\prime,a_4))\Big\|^2\nonumber\allowdisplaybreaks\\
   &=\Big\|\sum_{s^\prime}p(s^\prime |s,\!1)(\max_{a_1} f_0(\pmb{\theta}_k;\pmb{\phi}(s^\prime,a_1))-\max_{a_3} {f}(\pmb{\theta}_k;\pmb{\phi}(s^\prime,a_3))) \nonumber\allowdisplaybreaks\\
   & \qquad-\sum_{s^\prime}p(s^\prime |s,\!0)(\max_{a_2}f_0(\pmb{\theta}_k;\pmb{\phi}(s^\prime,a_2))-\max_{a_4} {f}(\pmb{\theta}_k;\pmb{\phi}(s^\prime,a_4)) \Big\|^2\nonumber\allowdisplaybreaks\\
   &\leq 2\|\max_{(s,a)}f_0(\pmb{\theta}_k;\pmb{\phi}(s,a))- {f}(\pmb{\theta}_k;\pmb{\phi}(s,a))\|^2\nonumber\allowdisplaybreaks\\
   &\leq 2\mathcal{O}\Big(\frac{c_1^3(\|\pmb{\theta}_0\|\!+\!|\lambda_0|\!+\!1)^3}{ m^{1/2}}\Big),
\end{align}
where the last inequality is due to Lemma \ref{lemma:approximation-gap-of-h}.
Substituting (\ref{eq:lemma6-2}) and (\ref{eq:lemma6-3}) back to (\ref{eq:lemma6}) yields the final results.

\end{proof}

\section{Proof of the Theorem \ref{thm:QW_convergence}}
To prove Theorem \ref{thm:QW_convergence}, we need the following three key lemmas about the error terms defined in (\ref{eq:residual-new}).

\begin{lemma}\label{lemma3}
 Let $\{\pmb{\theta}_k, {\lambda}_k\}$ be generated by (\ref{eq:ge_lambda_2TSA}). Then under 
Lemmas~\ref{lem:h_lipschitz}-\ref{lem:gh_monotone}, for any $k\geq \tau$, we have 
\begin{align} 
  \mathbb{E} \left[\Big\|\tilde{\pmb{\theta}}_{k+1}\Big\|^2|\mathcal{F}_{k-\tau}\right]&\leq (1+150\alpha_k^2+\eta_k/\alpha_k-2\alpha_k\mu_1)\mathbb{E}\Big[\left\|\hat{\pmb{\theta}}_{k}\right\|^2|\mathcal{F}_{k-\tau}\Big]+6\alpha_k^2\mathbb{E}\Big[\Big\|\hat{\lambda}_k\Big\|^2|\mathcal{F}_{k-\tau}\Big]\nonumber\allowdisplaybreaks\\
    &+\frac{\alpha_k^3}{\eta_k}\mathcal{O}\Big(c_1^3(\|\pmb{\theta}_0\|+|\lambda_0|+1)^3\cdot m^{-1/2}\Big). \label{eq:theta_hat}
\end{align}
\end{lemma}

\begin{proof}
According to (\ref{eq:residual-new}), we have $\hat{\pmb{\theta}}_{k+1}:=\pmb{\theta}_{k+1}-{\pmb{\theta}}_0^*=\hat{\pmb{\theta}}_{k}+\alpha_kh(X_k,\pmb{\theta}_k,\lambda_k),$
 which leads to 
\begin{align}
    \left\|\hat{\pmb{\theta}}_{k+1}\right\|^2&=\Big\|\hat{\pmb{\theta}}_k\Big\|^2+2\alpha_k\hat{\pmb{\theta}}_k^\intercal h(X_k,\pmb{\theta}_k,{\lambda}_k)+\Big\|\alpha_kh(X_k,\pmb{\theta}_k,{\lambda}_k)\Big\|^2\nonumber\displaybreak[1]\\
    &=\Big\|\hat{\pmb{\theta}}_{k}\Big\|^2+2\alpha_k\hat{\pmb{\theta}}_{k}^\intercal (h(X_k,\pmb{\theta}_k,\lambda_k)-h_0(X_k,\pmb{\theta}_k,\lambda_k))+2\alpha_k\hat{\pmb{\theta}}_{k}^\intercal h_0(X_k,\pmb{\theta}_k,\lambda_k)\nonumber\allowdisplaybreaks\\
    &\qquad +\alpha_k^2\|h(X_k,\pmb{\theta}_k,\lambda_k)-h_0(X_k,\pmb{\theta}_k,\lambda_k)+h_0(X_k,\pmb{\theta}_k,\lambda_k)\|^2\nonumber\allowdisplaybreaks\\
    &\leq\Big\|\hat{\pmb{\theta}}_{k}\Big\|^2+2\alpha_k\hat{\pmb{\theta}}_{k}^\intercal (h(X_k,\pmb{\theta}_k,\lambda_k)-h_0(X_k,\pmb{\theta}_k,\lambda_k))+2\alpha_k\hat{\pmb{\theta}}_{k}^\intercal h_0(X_k,\pmb{\theta}_k,\lambda_k)\nonumber\allowdisplaybreaks\\
    &\qquad+2\alpha_k^2\|h(X_k,\pmb{\theta}_k,\lambda_k)-h_0(X_k,\pmb{\theta}_k,\lambda_k)\|^2+2\alpha_k^2\|h_0(X_k,\pmb{\theta}_k,\lambda_k)\|^2.
\end{align}
The above inequality holds due to the fact that $\|\vx+\vy\|^2\leq 2\|\vx\|^2+2\|\vy\|^2$. 
 Taking expectations of $\|\hat{\pmb{\theta}}_{k+1}\|^2$ w.r.t $\mathcal{F}_{k-\tau}$ yields
\begin{align}
    \label{eq:40}\mathbb{E}\Big[\|\hat{\pmb{\theta}}_{k+1}\|^2|\mathcal{F}_{k-\tau}\Big]
    & {\leq}\mathbb{E}\Big[\Big\|\hat{\pmb{\theta}}_k\Big\|^2|\mathcal{F}_{k-\tau}\Big]+2\alpha_k\mathbb{E}\Big[\hat{\pmb{\theta}}_k^\intercal h_0(X_k,\pmb{\theta}_k,{\lambda}_k)|\mathcal{F}_{k-\tau}\Big]\nonumber\allowdisplaybreaks\\
    &\qquad\qquad+\underset{\text{Term}_1}{\underbrace{2\alpha_k^2\mathbb{E}\Big[\Big\|h_0(X_k,\pmb{\theta}_k,{\lambda}_k)\Big\|^2|\mathcal{F}_{k-\tau}\Big]}}\nonumber\displaybreak[3]\\
    &\qquad\qquad+\underset{\text{Term}_2}{\underbrace{2\alpha_k\mathbb{E}\Big[\hat{\pmb{\theta}}_{k}^\intercal (h(X_k,\pmb{\theta}_k,\lambda_k)-h_0(X_k,\pmb{\theta}_k,\lambda_k))|\mathcal{F}_{k-\tau}\Big]}}\nonumber\allowdisplaybreaks\\
    &\qquad\qquad+\underset{\text{Term}_3}{\underbrace{2\alpha_k^2\mathbb{E}\Big[\Big\|h(X_k,\pmb{\theta}_k,\lambda_k)-h_0(X_k,\pmb{\theta}_k,\lambda_k)\Big\|^2|\mathcal{F}_{k-\tau}\Big]}}\nonumber\allowdisplaybreaks\\
    &{\leq} \mathbb{E}\Big[\Big\|\hat{\pmb{\theta}}_k\Big\|^2|\mathcal{F}_{k-\tau}\Big]\!-\!2\alpha_k\mu_1\mathbb{E}\Big[\Big\|\tilde{\pmb{\theta}}_k\Big\|^2|\mathcal{F}_{k-\tau}\Big]\!+\text{Term}_1\!+\text{Term}_2+\text{Term}_3,
\end{align}

where the last inequality is due to Lemma \ref{lem:gh_monotone}.
Next, we bound each individual term. $\text{Term}_1$ is bounded as
\begin{align}
    \text{Term}_1&=2\alpha_k^2\mathbb{E}\Big[\Big\|h_0(X_k,\pmb{\theta}_k,{\lambda}_k)\Big\|^2|\mathcal{F}_{k-\tau}\Big]\nonumber\\
    &\overset{(c_1)}{=}2\alpha_k^2\mathbb{E}\Big[\Big\|h_0(X_k,\pmb{\theta}_k,{\lambda}_k)-h_0(X_k,\pmb{\theta}_k,y_0(\pmb{\theta}_k))+h_0(X_k,\pmb{\theta}_k,y_0(\pmb{\theta}_k))\nonumber\\
    &\qquad\qquad-h_0(X_k,\pmb{\theta}_0^*,y_0(\pmb{\theta}_0^*))+h_0(X_k,\pmb{\theta}_0^*,y_0(\pmb{\theta}_0^*))-H_0(\pmb{\theta}_0^*,y_0(\pmb{\theta}_0^*))\Big\|^2|\mathcal{F}_{k-\tau}\Big]\nonumber\\
    &\overset{(c_2)}{\leq} 6\alpha_k^2\mathbb{E}\Big[\Big\|h_0(X_k,\pmb{\theta}_k,{\lambda}_k)-h_0(X_k,\pmb{\theta}_k,y_0(\pmb{\theta}_k))\Big\|^2|\mathcal{F}_{k-\tau}\Big]\nonumber\\
    &\qquad\qquad+6\alpha_k^2\mathbb{E}\Big[\Big\|h_0(X_k,\pmb{\theta}_k,y_0(\pmb{\theta}_k))-h_0(X_k,\pmb{\theta}_0^*,y_0(\pmb{\theta}_0^*))\Big\|^2|\mathcal{F}_{k-\tau}\Big]\nonumber\\
    &\qquad\qquad+6\alpha_k^2\mathbb{E}\Big[\Big\|h_0(X_k,\pmb{\theta}_0^*,y_0(\pmb{\theta}_0^*))-H_0(\pmb{\theta}_0^*,y_0(\pmb{\theta}_0^*))\Big\|^2|\mathcal{F}_{k-\tau}\Big]\nonumber\\
    &\overset{(c_3)}{\leq}6\alpha_k^2\mathbb{E}\Big[\Big\|\hat{\lambda}_k\Big\|^2|\mathcal{F}_{k-\tau}\Big]+150\alpha_k^2\mathbb{E}\Big[\Big\|\hat{\pmb{\theta}}_k\Big\|^2|\mathcal{F}_{k-\tau}\Big],
\end{align}
where $(c_1)$ holds due to $H_0(\pmb{\theta}_0,y_0(\pmb{\theta}_0^*))=0$, $(c_2)$ follows from the triangular inequality, and $(c_3)$ follows from the Lipschitz continuity of $h_0(X, \pmb{\theta}, \lambda)$ in Lemma \ref{lem:h_lipschitz}.

$\text{Term}_2$ is bounded as
\begin{align}\nonumber
   \text{Term}_2 &=2\alpha_k\mathbb{E}\Big[\hat{\pmb{\theta}}_{k}^\intercal (h(X_k,\pmb{\theta}_k,\lambda_k)-h_0(X_k,\pmb{\theta}_k,\lambda_k))|\mathcal{F}_{k-\tau}\Big]\allowdisplaybreaks\\
   &\overset{(c_4)}{\leq} \frac{\eta_k}{\alpha_k}\mathbb{E}\Big[\left\|\hat{\pmb{\theta}}_{k}\right\|^2|\mathcal{F}_{k-\tau}\Big]+\frac{\alpha_k^3}{\eta_k}\mathbb{E}\Big[\Big\|h(X_k,\pmb{\theta}_k,\lambda_k)-h_0(X_k,\pmb{\theta}_k,\lambda_k)\Big\|^2|\mathcal{F}_{k-\tau}\Big]\nonumber\allowdisplaybreaks\\
   &\overset{(c_5)}{\leq} \frac{\eta_k}{\alpha_k}\mathbb{E}\Big[\left\|\hat{\pmb{\theta}}_{k}\right\|^2|\mathcal{F}_{k-\tau}\Big]+\frac{\alpha_k^2}{\eta_k}\mathcal{O}\Big(c_1^3(\|\pmb{\theta}_0\|+|\lambda_0|+1)^3\cdot m^{-1/2}\Big),
\end{align}
where $(c_4)$ holds
due to the fact that $2\vx^\intercal \vy\leq \|\vx\|^2+\|\vy\|^2$ and  $(c_5)$ is due to Lemma \ref{lemma:approximation-gap-of-h}.

$\text{Term}_3$ is bounded as
\begin{align}
\text{Term}_3&=2\alpha_k^2\mathbb{E}\Big[\Big\|h(X_k,\pmb{\theta}_k,\lambda_k)-h_0(X_k,\pmb{\theta}_k,\lambda_k)\Big\|^2|\mathcal{F}_{k-\tau}\Big]\nonumber\allowdisplaybreaks\\
    &\overset{(c_6)}{\leq} 2\alpha_k^2\mathcal{O}\Big(c_1^3(\|\pmb{\theta}_0\|+|\lambda_0|+1)^3\cdot m^{-1/2}\Big),
\end{align}
where  $(c_6)$ comes from Lemma \ref{lemma:approximation-gap-of-h}. 
Substituting $\text{Term}_1$, $\text{Term}_2$, and $\text{Term}_3$ back into (\ref{eq:40}) leads to the desired result in (\ref{eq:theta_hat}), which is
\begin{align}
\mathbb{E}\Big[\|\hat{\pmb{\theta}}_{k+1}\|^2|\mathcal{F}_{k-\tau}\Big]
    &{\leq} \mathbb{E}\Big[\Big\|\hat{\pmb{\theta}}_k\Big\|^2|\mathcal{F}_{k-\tau}\Big]-2\alpha_k\mu_1\mathbb{E}\Big[\Big\|\tilde{\pmb{\theta}}_k\Big\|^2|\mathcal{F}_{k-\tau}\Big]\nonumber\allowdisplaybreaks\\
    &+6\alpha_k^2\mathbb{E}\Big[\Big\|\hat{\lambda}_k\Big\|^2|\mathcal{F}_{k-\tau}\Big]+150\alpha_k^2\mathbb{E}\Big[\Big\|\hat{\pmb{\theta}}_k\Big\|^2|\mathcal{F}_{k-\tau}\Big]\nonumber\allowdisplaybreaks\\
    &+\frac{\eta_k}{\alpha_k}\mathbb{E}\Big[\left\|\hat{\pmb{\theta}}_{k}\right\|^2|\mathcal{F}_{k-\tau}\Big]+\frac{\alpha_k^3}{\eta_k}\mathcal{O}\Big(c_1^3(\|\pmb{\theta}_0\|+|\lambda_0|+1)^3\cdot m^{-1/2}\Big)\nonumber\allowdisplaybreaks\\
    &+2\alpha_k^2\mathcal{O}\Big(c_1^3(\|\pmb{\theta}_0\|+|\lambda_0|+1)^3\cdot m^{-1/2}\Big)\nonumber\allowdisplaybreaks\\
    &= (1+150\alpha_k^2+\eta_k/\alpha_k-2\alpha_k\mu_1)\mathbb{E}\Big[\left\|\hat{\pmb{\theta}}_{k}\right\|^2|\mathcal{F}_{k-\tau}\Big]+6\alpha_k^2\mathbb{E}\Big[\Big\|\hat{\lambda}_k\Big\|^2|\mathcal{F}_{k-\tau}\Big]\nonumber\allowdisplaybreaks\\
    &\qquad+(\alpha_k^3/\eta_k+2\alpha_k^2)\mathcal{O}\Big(c_1^3(\|\pmb{\theta}_0\|+|\lambda_0|+1)^3\cdot m^{-1/2}\Big)\nonumber.
\end{align}
By neglecting higher order infinitesimal, we have the inequality in (\ref{eq:theta_hat}). 
This completes the proof.
\end{proof}

\begin{lemma}\label{lemma2}
 Let $\{\pmb{\theta}_k, {\lambda}_k\}$ be generated by (\ref{eq:ge_lambda_2TSA}). Then under 
Lemmas~\ref{lem:h_lipschitz}-\ref{lem:gh_monotone}, for any $k\geq \tau$, we have 
\begin{align}\nonumber
     \mathbb{E} \left[\Big\|\hat{{\lambda}}_{k+1}\Big\|^2\Big|\mathcal{F}_{k-\tau}\right]&\leq (1-2\eta_k\mu_2+\alpha_k\eta_k+24\alpha_k^2+\frac{\eta_k}{\alpha_k}-\frac{2\eta_k^2\mu_k}{\alpha_k}+\eta_k^2+\frac{24\alpha_k^3}{\eta_k})\mathbb{E}\Big[\Big\|\hat{{\lambda}}_k\Big\|^2|\mathcal{F}_{k-\tau}\Big]\nonumber\allowdisplaybreaks\\
     &+(600\alpha_k^2+8\eta_k^2+\frac{8\eta_k^3}{\alpha_k}+\frac{600\alpha_k^3}{\eta_k})\mathbb{E}\Big[\Big\|\hat{\pmb{\theta}}_k\Big\|^2|\mathcal{F}_{k-\tau}\Big]\nonumber\\
    &+\frac{\eta_k}{\alpha_k}\mathcal{O}\Big(c_1^3(\|\pmb{\theta}_0\|+|\lambda_0|+1)^3\cdot m^{-1/2}\Big)\label{eq:Q_tilde}.
\end{align}
\end{lemma}

\begin{proof}
According to the definition in (\ref{eq:residual}), we have 
\begin{align*}
    \hat{\lambda}_{k+1}&=\lambda_{k+1}-y_0(\pmb{\theta}_{k+1})\nonumber\displaybreak[0]\\
    &=\hat{{\lambda}}_k+\eta_kg(\pmb{\theta}_k)+y_0(\pmb{\theta}_k)-y_0(\pmb{\theta}_{k+1}),
\end{align*}
which leads to
\begin{align}
    \Big\|\hat{{\lambda}}_{k+1}\Big\|^2\!\!&=\Big\|\hat{{\lambda}}_k+\eta_kg(\pmb{\theta}_k)+y_0(\pmb{\theta}_k)-y_0(\pmb{\theta}_{k+1})\Big\|^2\nonumber\displaybreak[0]\\
    &=\underset{\text{Term}_1}{\underbrace{\Big\|\hat{{\lambda}}_k+\eta_kg(\pmb{\theta}_k)\Big\|^2}}\!+\!\underset{\text{Term}_2}{\underbrace{\Big\|y_0(\pmb{\theta}_k)-y_0(\pmb{\theta}_{k+1})\Big\|^2}}\nonumber\displaybreak[1]\\
    &\qquad+\underset{\text{Term}_3}{\underbrace{2\left(\hat{{\lambda}}_k+\eta_kg(\pmb{\theta}_k)\right)\Big(y_0(\pmb{\theta}_k)-y_0(\pmb{\theta}_{k+1})\Big)}}.
\end{align}
The second equality is due to the fact that $\|\vx+\vy\|^2=\|\vx\|^2+\|\vy\|^2+2\vx^\intercal\vy$. We next analyze the conditional expectation of each term in $\Big\|\hat{{\lambda}}_{k+1}\Big\|^2$ on $\mathcal{F}_{k-\tau}$. We first focus on Term$_1$.

\begin{align*}
    \mathbb{E}&\Big[\text{Term}_1|\mathcal{F}_{k-\tau}\Big]\\
    &=\mathbb{E}\Big[\Big\|\hat{{\lambda}}_k\Big\|^2+2\eta_k\hat{{\lambda}}_k g(\pmb{\theta}_k)+\Big\|\eta_kg(\pmb{\theta}_k)\Big\|^2|\mathcal{F}_{k-\tau}\Big]\nonumber\displaybreak[1]\\
    &=\mathbb{E}\Big[\Big\|\hat{{\lambda}}_k\Big\|^2+2\eta_k\hat{{\lambda}}_k g_0(\pmb{\theta}_k)+2\eta_k\hat{{\lambda}}_k (g(\pmb{\theta}_k)-g_0(\pmb{\theta}_k))+\eta_k^2\Big\|g(\pmb{\theta}_k)-g_0(\pmb{\theta}_k)+g_0(\pmb{\theta}_k)\Big\|^2|\mathcal{F}_{k-\tau}\Big]\nonumber\displaybreak[1]\\
    &\leq\mathbb{E}\Big[\Big\|\hat{{\lambda}}_k\Big\|^2|\mathcal{F}_{k-\tau}\Big]+2\eta_k\mathbb{E}\Big[\hat{{\lambda}}_k g_0(\pmb{\theta}_k)|\mathcal{F}_{k-\tau}\Big]+2\eta_k\mathbb{E}\Big[\hat{{\lambda}}_k (g(\pmb{\theta}_k)-g_0(\pmb{\theta}_k))|\mathcal{F}_{k-\tau}\Big]\allowdisplaybreaks\\
    &\qquad\qquad+2\eta_k^2\mathbb{E}\Big[\Big\|g(\pmb{\theta}_k)-g_0(\pmb{\theta}_k)\Big\|^2|\mathcal{F}_{k-\tau}\Big]+2\eta_k^2\mathbb{E}\Big[\Big\|g_0(\pmb{\theta}_k)\Big\|^2|\mathcal{F}_{k-\tau}\Big]\nonumber\displaybreak[1]\\
    &\overset{(d_1)}{=}\mathbb{E}\Big[\Big\|\hat{{\lambda}}_k\Big\|^2|\mathcal{F}_{k-\tau}\Big]+2\eta_k\mathbb{E}\Big[\hat{{\lambda}}_k g_0(\pmb{\theta}_k)|\mathcal{F}_{k-\tau}\Big]+2\eta_k\mathbb{E}\Big[\hat{{\lambda}}_k (g(\pmb{\theta}_k)-g_0(\pmb{\theta}_k))|\mathcal{F}_{k-\tau}\Big]\allowdisplaybreaks\\
    &\qquad\qquad+2\eta_k^2\mathbb{E}\Big[\Big\|g(\pmb{\theta}_k)-g_0(\pmb{\theta}_k)\Big\|^2|\mathcal{F}_{k-\tau}\Big]+2\eta_k^2\mathbb{E}\Big[\Big\|g_0(\pmb{\theta}_k)-g_0(\pmb{\theta}_0^*)\Big\|^2|\mathcal{F}_{k-\tau}\Big]\nonumber\displaybreak[1]\\
    &\overset{(d_2)}{=}\mathbb{E}\Big[\Big\|\hat{{\lambda}}_k\Big\|^2|\mathcal{F}_{k-\tau}\Big]-2\eta_k\mu_2\mathbb{E}\Big[\Big\|\hat{{\lambda}}_k\Big\|^2|\mathcal{F}_{k-\tau}\Big]+8\eta_k^2\mathbb{E}\Big[\Big\|\hat{\pmb{\theta}}_k\Big\|^2|\mathcal{F}_{k-\tau}\Big]\nonumber\displaybreak[3]\\
    &\qquad\qquad+2\eta_k\mathbb{E}\Big[\hat{{\lambda}}_k (g(\pmb{\theta}_k)-g_0(\pmb{\theta}_k))|\mathcal{F}_{k-\tau}\Big]+2\eta_k^2\mathbb{E}\Big[\Big\|g(\pmb{\theta}_k)-g_0(\pmb{\theta}_k)\Big\|^2|\mathcal{F}_{k-\tau}\Big]\allowdisplaybreaks\\
    &\overset{(d_3)}{\leq} \mathbb{E}\Big[\Big\|\hat{{\lambda}}_k\Big\|^2|\mathcal{F}_{k-\tau}\Big]-2\eta_k\mu_2\mathbb{E}\Big[\Big\|\hat{{\lambda}}_k\Big\|^2|\mathcal{F}_{k-\tau}\Big]+8\eta_k^2\mathbb{E}\Big[\Big\|\hat{\pmb{\theta}}_k\Big\|^2|\mathcal{F}_{k-\tau}\Big]\nonumber\\
    &\qquad\qquad+\alpha_k\eta_k\mathbb{E}\Big[\Big\|\hat{{\lambda}}_k\Big\|^2|\mathcal{F}_{k-\tau}\Big]+(4\eta_k/\alpha_k+8\eta_k^2)\mathcal{O}\Big(c_1^3(\|\pmb{\theta}_0\|+|\lambda_0|+1)^3\cdot m^{-1/2}\Big),
\end{align*}
where $(d_1)$ follows from $g_0(\pmb{\theta}_0^*)=0$, $(d_2)$ holds due to Lemma \ref{lem:gh_monotone} and the Lipschitz continuity of $y_0$ in Lemma \ref{lem:y_lipschitz}, and $(d_3)$ comes from Lemma \ref{lemma:approximation-gap-of-h}.  For Term$_2$, we have

\begin{align}
    \mathbb{E}\Big[\text{Term}_2|\mathcal{F}_{k-\tau}\Big]&=\mathbb{E}\left[\Big\|y_0(\pmb{\theta}_k)-y_0(\pmb{\theta}_{k+1})\Big\|^2|\mathcal{F}_{k-\tau}\right]\nonumber\displaybreak[0]\\
    &=4\mathbb{E}\left[\Big\|\pmb{\theta}_k-\pmb{\theta}_{k+1}\Big\|^2|\mathcal{F}_{k-\tau}\right]\nonumber\displaybreak[1]\\
    &=4\alpha_k^2\mathbb{E}\Big[\Big\|h(X_k,\pmb{\theta}_k,{\lambda}_k)\Big\|^2|\mathcal{F}_{k-\tau}\Big]\nonumber\displaybreak[2]\\
    &=4\alpha_k^2\mathbb{E}\Big[\Big\|h(X_k,\pmb{\theta}_k,{\lambda}_k)-h_0(X_k,\pmb{\theta}_k,{\lambda}_k)+h_0(X_k,\pmb{\theta}_k,{\lambda}_k)\Big\|^2|\mathcal{F}_{k-\tau}\Big]\nonumber\\
    &=8\alpha_k^2\mathbb{E}\Big[\Big\|h_0(X_k,\pmb{\theta}_k,{\lambda}_k)\Big\|^2|\mathcal{F}_{k-\tau}\Big]+8\alpha_k^2\mathbb{E}\Big[\Big\|h(X_k,\pmb{\theta}_k,{\lambda}_k)-h_0(X_k,\pmb{\theta}_k,{\lambda}_k)\Big\|^2|\mathcal{F}_{k-\tau}\Big]\nonumber\\
    &\overset{(d_4)}{=}8\alpha_k^2\mathbb{E}\Big[\Big\|h_0(X_k,\pmb{\theta}_k,{\lambda}_k)-h_0(X_k,\pmb{\theta}_k,y_0(\pmb{\theta}_k))+h_0(X_k,\pmb{\theta}_k,y_0(\pmb{\theta}_k))\nonumber\\
    &\qquad\qquad-h_0(X_k,\pmb{\theta}_0^*,y_0(\pmb{\theta}_0^*))+h_0(X_k,\pmb{\theta}_0^*,y_0(\pmb{\theta}_0^*))-H_0(\pmb{\theta}_0^*,y_0(\pmb{\theta}_0^*))\Big\|^2|\mathcal{F}_{k-\tau}\Big]\nonumber\\
    &\qquad\qquad+8\alpha_k^2\mathbb{E}\Big[\Big\|h(X_k,\pmb{\theta}_k,{\lambda}_k)-h_0(X_k,\pmb{\theta}_k,{\lambda}_k)\Big\|^2|\mathcal{F}_{k-\tau}\Big]\nonumber\\
    &\overset{(d_5)}{\leq} 24\alpha_k^2\mathbb{E}\Big[\Big\|h_0(X_k,\pmb{\theta}_k,{\lambda}_k)-h_0(X_k,\pmb{\theta}_k,y_0(\pmb{\theta}_k))\Big\|^2|\mathcal{F}_{k-\tau}\Big]\nonumber\\
    &\qquad\qquad+24\alpha_k^2\mathbb{E}\Big[\Big\|h_0(X_k,\pmb{\theta}_k,y_0(\pmb{\theta}_k))-h_0(X_k,\pmb{\theta}_0^*,y_0(\pmb{\theta}_0^*))\Big\|^2|\mathcal{F}_{k-\tau}\Big]\nonumber\\
    &\qquad\qquad+24\alpha_k^2\mathbb{E}\Big[\Big\|h_0(X_k,\pmb{\theta}_0^*,y_0(\pmb{\theta}_0^*))-H_0(\pmb{\theta}_0^*,y_0(\pmb{\theta}_0^*))\Big\|^2|\mathcal{F}_{k-\tau}\Big]\nonumber\\
    &\qquad\qquad+8\alpha_k^2\mathbb{E}\Big[\Big\|h(X_k,\pmb{\theta}_k,{\lambda}_k)-h_0(X_k,\pmb{\theta}_k,{\lambda}_k)\Big\|^2|\mathcal{F}_{k-\tau}\Big]\nonumber\\
    &\overset{(d_6)}{\leq}24\alpha_k^2\mathbb{E}\Big[\Big\|\hat{\lambda}_k\Big\|^2|\mathcal{F}_{k-\tau}\Big]+600\alpha_k^2\mathbb{E}\Big[\Big\|\hat{\pmb{\theta}}_k\Big\|^2|\mathcal{F}_{k-\tau}\Big]\nonumber\\
    &\qquad\qquad+8\alpha_k^2\mathbb{E}\Big[\Big\|h(X_k,\pmb{\theta}_k,{\lambda}_k)-h_0(X_k,\pmb{\theta}_k,{\lambda}_k)\Big\|^2|\mathcal{F}_{k-\tau}\Big]\nonumber\allowdisplaybreaks\\
     &\overset{(d_7)}{\leq}24\alpha_k^2\mathbb{E}\Big[\Big\|\hat{\lambda}_k\Big\|^2|\mathcal{F}_{k-\tau}\Big]+600\alpha_k^2\mathbb{E}\Big[\Big\|\hat{\pmb{\theta}}_k\Big\|^2|\mathcal{F}_{k-\tau}\Big]\nonumber\\
    &\qquad\qquad+8\alpha_k^2\mathcal{O}\Big(c_1^3(\|\pmb{\theta}_0\|+|\lambda_0|+1)^3\cdot m^{-1/2}\Big)
\end{align}

where  $(d_4)$ is due to the fact that $H_0(\pmb{\theta}_0^*, y_0(\pmb{\theta}_0^*))=0$,
$(d_5)$ holds according to $\|\vx+\vy+\vz\|^2\leq 3\|\vx\|^2+3\|\vy\|^2+3\|\vz\|^2$ since  $g(X_k, f(\pmb{\lambda}^*), \pmb{\lambda}^*)=\pmb{0},$ 
$(d_6)$ holds because of the Lipschitz continuity of $h_0$ and $y_0$ in Lemma \ref{lem:h_lipschitz} and Lemma \ref{lem:y_lipschitz}, and $(d_7)$ comes from Lemma \ref{lemma:approximation-gap-of-h}.
Next, we have the conditional expectation of Term$_3$ as
\begin{align*}
\mathbb{E}\Big[\text{Term}_3|\mathcal{F}_{k-\tau}\Big]&=2\mathbb{E}\Big[\left\|\hat{{\lambda}}_k+\eta_kg(\pmb{\theta}_k)\right\|\cdot\Big\|y_0(\pmb{\theta}_k)-y_0(\pmb{\theta}_{k+1})\Big\||\mathcal{F}_{k-\tau}\Big]\allowdisplaybreaks\nonumber\\
    &\overset{(d_8)}{\leq} \frac{\eta_k}{\alpha_k}\text{Term}_1+\frac{\alpha_k}{\eta_k}\text{Term}_2 \allowdisplaybreaks\\
    &=\frac{\eta_k}{\alpha_k}\mathbb{E}\Big[\Big\|\hat{{\lambda}}_k\Big\|^2|\mathcal{F}_{k-\tau}\Big]-\frac{2\eta_k^2\mu_2}{\alpha_k}\mathbb{E}\Big[\Big\|\hat{{\lambda}}_k\Big\|^2|\mathcal{F}_{k-\tau}\Big]+\frac{8\eta_k^3}{\alpha_k}\mathbb{E}\Big[\Big\|\hat{\pmb{\theta}}_k\Big\|^2|\mathcal{F}_{k-\tau}\Big]\nonumber\\
    &\qquad\qquad+{\eta_k^2}\mathbb{E}\Big[\Big\|\hat{{\lambda}}_k\Big\|^2|\mathcal{F}_{k-\tau}\Big]+\frac{\eta_k}{\alpha_k}(4\eta_k/\alpha_k+8\eta_k^2)\mathcal{O}\Big(c_1^3(\|\pmb{\theta}_0\|+|\lambda_0|+1)^3\cdot m^{-1/2}\Big)\allowdisplaybreaks\\
    &+\frac{24\alpha_k^3}{\eta_k}\mathbb{E}\Big[\Big\|\hat{\lambda}_k\Big\|^2|\mathcal{F}_{k-\tau}\Big]+\frac{600\alpha_k^3}{\eta_k}\mathbb{E}\Big[\Big\|\hat{\pmb{\theta}}_k\Big\|^2|\mathcal{F}_{k-\tau}\Big]\nonumber\\
    &\qquad\qquad+\frac{8\alpha_k^3}{\eta_k}\mathcal{O}\Big(c_1^3(\|\pmb{\theta}_0\|+|\lambda_0|+1)^3\cdot m^{-1/2}\Big),
\end{align*}
where $(d_8)$ holds because $2\bold{x}^T\bold{y}\leq 1/\beta\|\bold{x}\|^2+\beta\|\bold{y}\|^2$, $\forall \beta>0$.
Summing $\text{Term}_1$, $\text{Term}_2$, and $\text{Term}_3$ and neglecting higher order infinitesimal yield the desired result.
\end{proof}

Now we are ready to prove the results in Theorem \ref{thm:QW_convergence}.
Providing Lemma~\ref{lemma2} and Lemma~\ref{lemma3},  if $\frac{\eta_{k}}{\alpha_k}$ is non-increasing, we have the following inequality
\begin{align}\label{eq:M_hat}
    \mathbb{E}\left[\hat{M}(\pmb{\theta}_{k+1}, \lambda_{k+1})\Big|\mathcal{F}_{k-\tau}\right]&=\mathbb{E}\left[\frac{\eta_k}{\alpha_k}\Big\|\hat{\pmb{\theta}}_{k+1}\Big\|^2+\Big\|\hat{{\lambda}}_{k+1}\Big\|^2\Big|\mathcal{F}_{k-\tau}\right]\nonumber\allowdisplaybreaks\\
    &\leq\frac{\eta_k}{\alpha_k}(1-2\alpha_k\mu_1)\mathbb{E}\Big[\left\|\hat{\pmb{\theta}}_{k}\right\|^2|\mathcal{F}_{k-\tau}\Big]+\frac{600\alpha_k^3}{\eta_k}\mathbb{E}\Big[\left\|\hat{\pmb{\theta}}_{k}\right\|^2|\mathcal{F}_{k-\tau}\Big]\nonumber\allowdisplaybreaks\\
    &\qquad\qquad+\frac{8\alpha_k^3}{\eta_k}\mathcal{O}\Big(c_1^3(\|\pmb{\theta}_0\|+|\lambda_0|+1)^3\cdot m^{-1/2}\Big)\nonumber\allowdisplaybreaks\\ 
    &\qquad\qquad+(1-2\eta_k\mu_2)\mathbb{E}\Big[\Big\|\hat{{\lambda}}_k\Big\|^2|\mathcal{F}_{k-\tau}\Big]\nonumber\allowdisplaybreaks\\
    &\qquad\qquad+\frac{600\alpha_k^3}{\eta_k}\mathbb{E}\Big[\left\|\hat{{\lambda}}_{k}\right\|^2|\mathcal{F}_{k-\tau}\Big].
\end{align}

Since $(k+1)^2\cdot\frac{\alpha_k^3}{\eta}=\frac{\alpha_0^3}{\eta_0}(k+1)^{1/3}$,
multiplying both sides of (\ref{eq:M_hat}) with $(k+1)^2$, we have 
\begin{align}\nonumber
    &(k+1)^2\mathbb{E}\Big[\hat{M}(\pmb{\theta}_{k+1},\lambda_{k+1})\Big|\mathcal{F}_{k-\tau}\Big]\\ \nonumber
    &\qquad\qquad{\leq} k^2\mathbb{E}\Big[\hat{M}(\pmb{\theta}_{k},\lambda_{k})\Big|\mathcal{F}_{k-\tau}\Big]+\frac{600\alpha_0^3}{\eta_0}(k+1)^{1/3}\left(\Big\|\hat{\pmb{\theta}}_k\Big\|^2+\Big\|\hat{{\lambda}}_k\Big\|^2\right)\nonumber\allowdisplaybreaks\\
    &\qquad\qquad\qquad+\frac{8\alpha_0^3}{\eta_0}(k+1)^{1/3}\mathcal{O}\Big(c_1^3(\|\pmb{\theta}_0\|+|\lambda_0|+1)^3\cdot m^{-1/2}\Big).   
   \label{eq: recursion}
\end{align}
Summing (\ref{eq: recursion}) from time step $\tau$ to time step $k$, we have 
\begin{align}\nonumber
   (k+1)^2\mathbb{E}\Big[\hat{M}(\pmb{\theta}_{k+1},\lambda_{k+1})\Big|\mathcal{F}_k\Big]
     &\leq \tau^2\mathbb{E}\Big[\hat{M}(\pmb{\theta}_\tau,\lambda_\tau)\Big]+\frac{600\alpha_0^3}{\eta_0}(k+1)^{4/3}\left(\Big\|\hat{\pmb{\theta}}_\tau\Big\|^2+\Big\|\hat{{\lambda}}_\tau\Big\|^2\right)\\
     &+\frac{8\alpha_0^3}{\eta_0}(k+1)^{4/3}\mathcal{O}\Big(c_1^3(\|\pmb{\theta}_0\|+|\lambda_0|+1)^3\cdot m^{-1/2}\Big)\nonumber\\
    &\leq \tau^2\mathbb{E}\Big[\hat{M}(\pmb{\theta}_\tau,\lambda_\tau)\Big]+\frac{600\alpha_0^3}{\eta_0}\frac{(C_1+\|\hat{\pmb{\theta}}_0\|)^2+(2C_1+\|\hat{\lambda}_0\|)^2}{(k+1)^{-4/3}}\nonumber\\
    &+\frac{8\alpha_0^3}{\eta_0}\frac{\mathcal{O}\Big(c_1^3(\|\pmb{\theta}_0\|+|\lambda_0|+1)^3\cdot m^{-1/2}\Big)}{(k+1)^{-4/3}},
\end{align}
where the second inequality holds due to Lemma \ref{lem:bounded_parameter}.
Finally, dividing both sides by $(k+1)^2$ and moving the constant term into $\mathcal{O}(\cdot)$ yields the results in Theorem \ref{thm:QW_convergence}.

\section{Auxiliary Lemmas}
In this part, we present several key lemmas which are needed for the major proofs.
We first show the parameters update in (\ref{eq:ge_lambda_2TSA}) is bounded in the following lemma. 
\begin{lemma}\label{lem:bounded_parameter}
The update of ~ $\pmb{\theta}_k$ and $\lambda_k$ in (\ref{eq:ge_lambda_2TSA}) is bounded with respect to the initial $\pmb{\theta}_0$ and $\lambda_0$, i.e., $$\|\pmb{\theta}_k-\pmb{\theta}_0\|+|\lambda_k-\lambda_0|\leq c_1(\|\pmb{\theta}_0\|+|\lambda_0|+1),$$
with $c_1$ be the constant, i.e., $c_1:=\frac{1}{2}+\frac{3}{2}(L_h^\prime\alpha_\tau+L_g^\prime\eta_\tau)(L_h^\prime\alpha_{\tau}+L_g^\prime\eta_{\tau}+1)$. 
\end{lemma}
\begin{proof}

Without loss of generality, we assume that $$L_h^\prime\geq\max(3,\max_{X\in\cX}\|h_0(X,0,0)\|),~ L_g^\prime\geq\max(2,\max_{X\in\cX}\|g_0(X,0,0)\|).$$  Then based on triangular inequality and Lemmas \ref{lem:h_lipschitz}-\ref{lem:g_lipschitz},  we have
\begin{align}\label{eq:hg_lip}
    \|h_0(X,\pmb{\theta},{\lambda})\|\leq L_h^\prime(\|\pmb{\theta}\|+|{\lambda}|+1), ~\|g_0(X,\pmb{\theta},{\lambda})\|\leq L_g^\prime(\|\pmb{\theta}\|+|{\lambda}|+1), \forall \pmb{\theta}, {\lambda}, X\in\cX.  
\end{align}

Since we have $\pmb{\theta}_{k+1}=\pmb{\theta}_{k}+\alpha_kh(X_k,\pmb{\theta}_k, \lambda_k)$, we have the following inequality due to Lipschitz continuity of $h$ in (\ref{eq:hg_lip})
\begin{align}\label{eq:theta}
    \|\pmb{\theta}_{k+1}-\pmb{\theta}_{k}\|=\alpha_k\|h(X_k,\pmb{\theta}_k, \lambda_k)\|\leq \alpha_kL_h^\prime(\|\pmb{\theta}_k\|+|\lambda_k|+1).
\end{align}
Similarly, we have
\begin{align}\label{eq:lambda}
    |\lambda_{k+1}-\lambda_k|=\eta_k|g(X_k,\pmb{\theta}_k, \lambda_k)|\leq \eta_kL_g^\prime(\|\pmb{\theta}_k\|+|\lambda_k|+1).
\end{align}
Due to triangular inequality, adding (\ref{eq:theta}) and (\ref{eq:lambda}) leads to
\begin{align}\nonumber
    \|\pmb{\theta}_{k+1}\|+|\lambda_{k+1}|+1&\leq (L_h^\prime\alpha_k+L_g^\prime\eta_k+1)(\|\pmb{\theta}_k\|+|\lambda_k|+1)\allowdisplaybreaks\\
    &\leq(L_h^\prime\alpha_0+L_g^\prime\eta_0+1)(\|\pmb{\theta}_k\|+|\lambda_k|+1),\label{eq:theta+lambda}
\end{align}
where the second inequality holds due to the non-increasing learning rates $\{\alpha_k, \eta_k\}$. Rewriting the above inequality in (\ref{eq:theta+lambda}) in a recursive manner yields
\begin{align}
  \|\pmb{\theta}_{k}\|+\lambda_{k}+1 \leq  (L_h^\prime\alpha_0+L_g^\prime\eta_0+1)^{k-\tau}(\|\pmb{\theta}_\tau\|+|\lambda_\tau|+1).
\end{align}
Hence, we have
\begin{align*}
    \|\pmb{\theta}_k-\pmb{\theta}_{k-\tau}\|+|\lambda_k-\lambda_{k-\tau}|&\leq \sum_{t=k-\tau}^{k-1}\|\pmb{\theta}_{t+1}-\pmb{\theta}_{t}\|+|\lambda_{t+1}-\lambda_t|\allowdisplaybreaks\\
    &\leq (L_h^\prime\alpha_0+L_g^\prime\eta_0)\sum_{t=k-\tau}^{k-1}(\|\pmb{\theta}_t\|+|\lambda_t|+1)\allowdisplaybreaks\\
    &\leq (L_h^\prime\alpha_0+L_g^\prime\eta_0)(\|\pmb{\theta}_{k-\tau}\|+|\lambda_{k-\tau}|+1)\sum_{t=k-\tau}^{k-1}(L_h^\prime\alpha_0+L_g^\prime\eta_0+1)^{t-\tau}\allowdisplaybreaks\\
    &=[(L_h^\prime\alpha_0+L_g^\prime\eta_0+1)^{\tau}-1](\|\pmb{\theta}_{k-\tau}\|+|\lambda_{k-\tau}|+1)\allowdisplaybreaks\\
    &\leq (e^{(L_h^\prime\alpha_0+L_g^\prime\eta_0)\tau}-1)(\|\pmb{\theta}_{k-\tau}\|+|\lambda_{k-\tau}|+1)\allowdisplaybreaks\\
    &\leq 2(L_h^\prime\alpha_0+L_g^\prime\eta_0)\tau(\|\pmb{\theta}_{k-\tau}\|+|\lambda_{k-\tau}|+1),
\end{align*}
where the last inequality holds when 
$(L_h^\prime\alpha_0+L_g^\prime\eta_0)\tau\leq 1/4$. This implies when $k=\tau$, we have
\begin{align}
  \|\pmb{\theta}_\tau-\pmb{\theta}_{0}\|+|\lambda_\tau-\lambda_{0}|\leq 2(L_h^\prime\alpha_0+L_g^\prime\eta_0)\tau (\|\pmb{\theta}_{0}\|+|\lambda_{0}|+1).
\end{align}

Similarly, we also have
\begin{align}
    \|\pmb{\theta}_k-\pmb{\theta}_{\tau}\|+\lambda_k-\lambda_\tau&\leq \sum_{t=\tau}^{k-1}\|\pmb{\theta}_{t+1}-\pmb{\theta}_{t}\|+\lambda_{t+1}-\lambda_t\nonumber\allowdisplaybreaks\\
    &\leq \sum_{t=\tau}^{k-1}(L_h^\prime\alpha_t+L_g^\prime\eta_t)(\|\pmb{\theta}_t\|+\lambda_t+1)\nonumber\allowdisplaybreaks\\
    &\leq (\|\pmb{\theta}_\tau\|+\lambda_\tau+1)\sum_{t=\tau}^{k-1}(L_h^\prime\alpha_t+L_g^\prime\eta_t)\prod_{i=0}^{t-\tau}(L_h^\prime\alpha_{\tau+i}+L_g^\prime\eta_{\tau+i}+1).
\end{align}

Therefore, the following inequality holds
\begin{align*}
    &\|\pmb{\theta}_k-\pmb{\theta}_0\|+|\lambda_k-\lambda_0|\nonumber\allowdisplaybreaks\\
    &\leq \|\pmb{\theta}_k-\pmb{\theta}_{\tau}\|+|\lambda_k-\lambda_\tau|+ \|\pmb{\theta}_\tau-\pmb{\theta}_{0}\|+|\lambda_\tau-\lambda_{0}|\nonumber\allowdisplaybreaks\\
    &\leq 2(L_h^\prime\alpha_0+L_g^\prime\eta_0)\tau (\|\pmb{\theta}_{0}\|+|\lambda_{0}|+1)\nonumber\allowdisplaybreaks\\
    &\qquad+(\|\pmb{\theta}_\tau\|+\lambda_\tau+1)\sum_{t=\tau}^{k-1}(L_h^\prime\alpha_t+L_g^\prime\eta_t)\prod_{i=0}^{t-\tau}(L_h^\prime\alpha_{\tau+i}+L_g^\prime\eta_{\tau+i}+1)\nonumber\allowdisplaybreaks\\
    &\leq 2(L_h^\prime\alpha_0+L_g^\prime\eta_0)\tau (\|\pmb{\theta}_{0}\|+|\lambda_{0}|+1)\nonumber\allowdisplaybreaks\\
    &\qquad+(2(L_h^\prime\alpha_0+L_g^\prime\eta_0)\tau +1)\sum_{t=\tau}^{k-1}(L_h^\prime\alpha_t+L_g^\prime\eta_t)\prod_{i=0}^{t-\tau}(L_h^\prime\alpha_{\tau+i}+L_g^\prime\eta_{\tau+i}+1)(\|\pmb{\theta}_{0}\|+|\lambda_{0}|+1)\nonumber\allowdisplaybreaks\\
    &\leq \left(\frac{1}{2}+\frac{3}{2}\sum_{t=\tau}^{k-1}(L_h^\prime\alpha_t+L_g^\prime\eta_t)\prod_{i=0}^{t-\tau}(L_h^\prime\alpha_{\tau+i}+L_g^\prime\eta_{\tau+i}+1)\right)(\|\pmb{\theta}_{0}\|+|\lambda_{0}|+1),
\end{align*}
with the last equality holds when $(L_h^\prime\alpha_0+L_g^\prime\eta_0)\tau\leq 1/4$. When $\sum_{t=\tau}^{k-1}(L_h^\prime\alpha_t+L_g^\prime\eta_t)\prod_{i=0}^{t-\tau}(L_h^\prime\alpha_{\tau+i}+L_g^\prime\eta_{\tau+i}+1)$ is non-increasing with $k$, then we can set $c_1$ as 
$c_1:=\frac{1}{2}+\frac{3}{2}.(L_h^\prime\alpha_\tau+L_g^\prime\eta_\tau)(L_h^\prime\alpha_{\tau}+L_g^\prime\eta_{\tau}+1)$.
This completes the proof. 
\end{proof}

Provided Lemma \ref{lem:bounded_parameter}, we have the following lemma related with local linearization of Q functions and the original Q functions. 
\begin{lemma}[Lemma 5.2 in \cite{cai2023neural}]\label{lemma:approximation-gap-of-h}
There exists a constant $c_1$such that 
\begin{align*}
    \mathbb{E}\Big[\|h(X_k,\pmb{\theta}_k,\lambda_k)-h_0(X_k,\pmb{\theta}_k,\lambda_k)\|^2|\mathcal{F}_{k-\tau}\Big]\leq \mathcal{O}\Big(c_1^3(\|\pmb{\theta}_0\|+|\lambda_0|+1)^3\cdot m^{-1/2}\Big).
\end{align*}
\end{lemma}

Lemma \ref{lemma:approximation-gap-of-h} indicates that if the updated parameter is always bounded in a ball with the initialized one as the center and a fixed radius, the local linearized function $f_0(\cdot)$ in (\ref{eq:neural_Q_linear}) and the original neural network approximated function $f(\cdot)$ in (\ref{eq:neural_Q}) have bounded gap, which tends to be zero as the width of hidden layer $m$ grow large. For interested readers, please refer to \cite{cai2023neural} for detailed proofs of this lemma.

\end{document}